\documentclass{article}

\usepackage{microtype}
\usepackage{mathrsfs}
\usepackage{graphicx}
\usepackage{subfigure}
\usepackage{wrapfig}
\usepackage{booktabs} % for professional tables

\usepackage{hyperref}

\usepackage[accepted]{icml2024}

\usepackage{amsmath}
\usepackage{amssymb}
\usepackage{mathtools}
\usepackage{amsthm}
\usepackage{booktabs}
\usepackage{preamble}
\usepackage{notation}
\usepackage{thm-restate}
\usepackage{thmtools}

\allowdisplaybreaks % Just added, to be able to split the long equations.

\usepackage[capitalize,noabbrev]{cleveref}

\icmltitlerunning{Incentivized Learning in Bandit Games}

\begin{document}

\twocolumn[
\icmltitle{Incentivized Learning in Principal-Agent Bandit Games}



\icmlsetsymbol{equal}{*}

\begin{icmlauthorlist}
\icmlauthor{Antoine Scheid}{CMAP}
\icmlauthor{Daniil Tiapkin}{CMAP,CNRSSACLAY}
\icmlauthor{Etienne Boursier}{INRIASACLAY}
\icmlauthor{Aymeric Capitaine}{CMAP}
\icmlauthor{El Mahdi El Mhamdi}{CMAP}
\icmlauthor{Eric Moulines}{CMAP}
\icmlauthor{Michael I. Jordan}{UCBERKELEY,INRIAPARIS}
\icmlauthor{Alain Durmus}{CMAP}
\end{icmlauthorlist}

\icmlaffiliation{CMAP}{Centre de Mathématiques Appliquées – CNRS – École polytechnique – Institut Polytechnique de
Paris – route de Saclay 91128 Palaiseau cedex}
\icmlaffiliation{CNRSSACLAY}{Université Paris-Saclay, CNRS, Laboratoire de mathématiques d'Orsay, 91405, Orsay, France}
\icmlaffiliation{INRIASACLAY}{INRIA, Universite Paris Saclay, LMO, Orsay, France}
\icmlaffiliation{UCBERKELEY}{University of California, Berkeley}
\icmlaffiliation{INRIAPARIS}{Inria, Ecole Normale Sup\'erieure, PSL Research University}
%\icmlaffiliation{comp}{Company Name, Location, Country}
%\icmlaffiliation{sch}{School of ZZZ, Institute of WWW, Location, Country}

\icmlcorrespondingauthor{Antoine Scheid}{antoine.scheid@polytechnique.edu}
%\icmlcorrespondingauthor{Firstname2 Lastname2}{first2.last2@www.uk}

\icmlkeywords{Bandits, Incentives, Contextual Bandits, Decision-Making under Uncertainty, Mechanism Design, Principal-Agent}

\vskip 0.3in
]



\printAffiliationsAndNotice{}  % leave blank if no need to mention equal contribution

\begin{abstract}
\looseness=-1 This work considers a repeated principal-agent bandit game, where the principal can only interact with her environment through the agent. The principal and the agent have misaligned objectives and the choice of action is only left to the agent. However, the principal can influence the agent's decisions by offering incentives which add up to his rewards. The principal aims to iteratively learn an incentive policy to maximize her own total utility. This framework extends usual bandit problems and is motivated by several practical applications, such as healthcare or ecological taxation, where traditionally used mechanism design theories often overlook the learning aspect of the problem. We present nearly optimal (with respect to a horizon $T$) learning algorithms for the principal's regret in both multi-armed and linear contextual settings. Finally, we support our theoretical guarantees through numerical experiments.
\end{abstract}

\section{Introduction}

\looseness=-1 Decision-making under uncertainty is a ubiquitous feature of real-world applications of machine learning, arising in domains as diverse as recommendation systems \cite{li2010contextual}, healthcare \cite{yu2021reinforcement}, and agriculture \cite{evans2017data}.  Multi-armed bandits provide a classical point of departure for decision-making under uncertainty in these settings \cite{thompson1933likelihood,woodroofe1979one, lattimore2020bandit, slivkins2019introduction}. The basic bandit solution involves an agent who learns which decisions yield high rewards via repeated experimentation. Real-world decision-making problems, however, often present challenges that are not addressed in this simple optimization framework.  These include the challenge of scarcity when there are multiple decision-makers, issues of misaligned objectives, and problems arising from information asymmetries and signaling.  The economics literature addresses these issues through the design of game-theoretic mechanisms, including auctions and contracts \citep[see, e.g.,][]{myerson1989mechanism, laffont2009theory}, aiming to achieve favorable outcomes despite agents' self-interest and limited information set. Unfortunately, the economics literature tends to neglect the learning aspect of the problem, often assuming that preferences, or distributions on preferences, are known a priori.  Our work focuses on the blend of mechanism design and learning.  We study a principal-agent model with information asymmetry and we develop a learning framework in which the principal aims to uncover the true preferences of the agent while optimizing her own gains.

Building on the work of \citet[][]{dogan2023estimating,dogan2023repeated}, we
consider a repeated game between a principal and an agent, where, at each round, the principal proposes an incentive transfer associated with any action. The agent greedily chooses the action that maximizes the sum of his expected reward and the incentive. The goal of the principal is to learn an incentive policy which maximizes her own utility over time, taking into account both the rewards that she reaps and the costly incentives that she offers.

Our contributions are as follows:
\vspace{-0.35cm}
\begin{itemize}[itemsep=-2pt, leftmargin=8pt]
    \item \looseness=-1 We present the \texttt{Incentivized Principal-Agent Algorithm} (\texttt{IPA}) framework, which comprises two steps. First, \texttt{IPA} estimates the minimal level of incentive needed to make the agent select any desired action. Subsequently, forming an upper estimate of these incentives, $\algU$ uses a regret-minimization algorithm in a black-box fashion. The overall algorithm achieves both nearly optimal distribution-free and instance-dependent regret bounds.
    \item We extend $\algU$ to  the linear contextual bandit setting \citep[see, e.g.,][]{abe1999associative, auer2002using, dani2008stochastic}, significantly broadening its applicability in various applications. Here $\calgU$ achieves a $\cO(d \sqrt{T} \log(T))$ regret bound. We emphasize that $\calgU$ is the first known algorithm for incentivized learning in a contextual setting. Moreover it matches, up to logarithmic factors, the  minimax lower bound $\Omega(d \sqrt{T})$ for the easier problem of stochastic linear bandits~\citep{rusmevichientong2010linearly}.
\end{itemize}

\section{Related Work}

While classical work on bandit problems and reinforcement learning has predominantly focused on single-agent scenarios, many emerging applications require considering  multiple agents.  Recent literature has accordingly begun to study frameworks for learning in multi-agent multi-armed bandit settings \citep[see, e.g.,][]{boursier2022survey}. 

\citet{mansour2020bayesian} discuss how a social planner can simultaneously learn and influence self-interested agents' decisions through Bayesian-Incentive Compatible (BIC) recommendations. The rationale behind this notion is that a BIC recommendation guarantees to each agent a maximal reward given the past, at any step. The social planner objective is to design BIC recommendations that maximize the global welfare. \citet{mansour2020bayesian} propose an algorithm for solving this problem in both multi-armed and contextual bandit settings. Notably, their work turns any black-box bandit algorithm into a BIC algorithm. For this problem, \citet{sellke2021price} show that Thompson sampling can be made BIC, with a sufficient number of initial observations. \citet{hu2022incentivizing} extend this work to the combinatorial bandits problem.

Another line of work due to  \citet{banihashem2023bandit} and \citet{simchowitz2023exploration} studies how a principal can provide recommendations to agents so that they explore all reachable states in a Markov Decision Process (MDP). To this end, the principal supplies the agents with a modified history, with the modifications carefully chosen to retain the agents' trust. This line of work is closely related to the online Bayesian persuasion literature \citep[see, e.g.,][]{castiglioni2020online}, which dates to the seminal work of \citet{kamenica2011bayesian}. Online Bayesian persuasion consists of the principal sequentially influencing agents' decision with signals in her own interest.

In these works, the information asymmetry favors the principal, so that the principal can influence the agent's decision at little or no cost. In our problem, the agent instead has perfect knowledge of the problem parameters and his action can only be influenced through utility transfers.

\looseness=-1
\citet{ben2023principal} study a principal and an agent sharing a common Markov Decision Process (MDP) with different reward functions. Similarly to our setting, at each step the action is chosen by an agent with full knowledge of the game. The objective of the principal is to minimize her cumulative regret under a constraint on the incentive budget. Despite extending our setup to MDPs, \citet{ben2023principal} do not consider uncertainty in the principal's side, turning the game into an optimization problem.

\looseness=-1
Related issues arise in the study of dynamic pricing \citep{den2015dynamic, javanmard2019dynamic, mao2018contextual, golrezaei2023incentive}. Our work diverges from dynamic pricing in that in our case the principal not only faces uncertainty with respect to the agent's utility, but also with respect to her own utility.

\looseness=-1
Some works study similar principal-agent games but with a specific focus on the achievable optimality of the contract \citet{cohen2022learning} or a specific stochastic model for the agent's behavior \citet{conitzer2006learning}.

\looseness=-1 Finally, our study is inspired by the work of \citet{dogan2023repeated} to explore the principal's learning mechanism within a principal-agent setting. They propose an $\epsilon$-Greedy algorithm with suboptimal regret guarantees. In particular, it suffers an exponential dependence in the number of actions. In contrast, we provide both distribution free and instance-dependent regret bounds that nearly match the known lower bounds. Also, we extend our approach to the non-trivial contextual case. Finally, \citet{dogan2023estimating} extend \citep{dogan2023repeated}, taking into account presence of uncertainty on the agent's side. 
\section{Multi-Armed Principal-Agent Learning}\label{sec:IPA}

\paragraph{Setup.} \looseness=-1 We consider a repeated principal-agent game. A contextual version of the game is introduced in \cref{sec:contextual}. The action set for the agent (or set of arms) is fixed to be $\cA\coloneqq [K]=\{1, \ldots, K\}, K\in \N^\star$. We assume that the agent's rewards, $\s=(\s_1, \ldots, \s_K)\in \R_+^K$, are deterministic and that they are known to the agent and unknown to the principal.

For each action $a \in [K]$, the rewards of the principal are given by a random, \iid\ sequence $(X_a(t))_{t\in [T]}$, where $X_a(t) \sim \nu_a$ and $\nu_a$ is the arm distribution. The distributions $\{\nu_a\}_{a=1}^K$ are unknown to the principal and are learned as a consequence of the following principal-agent interaction.

At each round $t\in[T]$, where $T$ is the game horizon, the principal proposes an incentive $\icv(t) \in \R_+$ associated with an action $a_t \in [K]$. The agent then greedily chooses action $A_t$ maximizing his utility:
\begin{equation}
\label{eq:def_A_t}
A_t \in \argmax_{a \in [K]} \{ \s_a + \indi{a_t}(a)\icv(t)\} \eqsp,
\end{equation}
breaking ties arbitrarily.\footnote{Note that the related works \citep{ben2023principal,simchowitz2023exploration} assume a tie-breaking in favor of the principal, an assumption that we do not need here.} 
The principal then observes the arm $A_t$ selected by the agent, as well as her reward given by $X_{A_t}(t)$. The utility of the principal on the round is given by $X_{A_t}(t) -\indi{a_t}(A_t)\icv(t)$. For any $a \in [K]$, the principal's mean reward is $\theta_a \coloneqq \E\parentheseDeux{X_a(t)}$. See \Cref{table:notations} in \Cref{appendix:notations} for a summary of the main definitions used in this section.

The sequence of incentives $(a_t, \icv(t))_{t\in [T]}$ defines a sequence of actions $(A_t)_{t\in [T]}$ chosen by the agent. The goal of the principal is to maximize her total utility. On a single round, she thus aims at proposing an optimal incentive $\icv^\optimal$ on an arm $a^\optimal \in [K]$, which solves
\begin{equation}
\label{eq:def_objective_princple_0}
\begin{gathered}
 \textstyle \text{maximize} \,\,  \int x \nu_a(\rmd x)-\icv \text{ over } \icv \in \R_+,a \in [K] \\ \text{ such that } \, a \in \argmax_{a' \in [K]}\left\{ \s_{a'} + \1_{a}(a')\icv \right\} \eqsp.
\end{gathered}
\end{equation}
This is consistent with the conventional framework for utility in bandit problems, where we subtract the cost of incentives to the principal. Here, the principal's influence is exerted solely through the strategic use of incentives, carefully designed to guide the agent's behavior. We define $\mu^\star \coloneqq \int x \nu_{a^\optimal}(\rmd x) - \icv^\optimal$. Maximizing the total utility of the principal over $T$ rounds is equivalent to minimizing the expected regret, defined as
\begin{equation}
\label{equation:regret_definition_non_contextual}
\regret(T) \coloneqq T\, \mu^\star  - \sum_{t=1}^T \E\parentheseDeux{X_{A_t}(t) -\indi{a_t}(A_t) \icv(t) } \eqsp.
\end{equation}

\textbf{Remark.} In the prior work of~\citet{dogan2023repeated}, incentives were defined as a vector of size $K$, where the incentivize associated with an action $a\in [K]$ was denoted $\icv_a$. In our setting, since the goal of the principal is to make sure that the agent picks one prescribed action, it is enough to consider a restricted family of the form $\icv_a = \indi{a_t}(a) \icv(t)$, where $(a_t, \icv(t))$ are incentives in the sense defined above.

We make the following assumption.
\begin{assumption}
\label{assumption:non_contextual_bounded_reward_principal}
For any $a\in [K]$, $\s_a \in [0,1]$.
\end{assumption}
Neither the distributions $\nu_a$ nor the preferences of the agent are known to the principal. Another difficulty arises from designing the magnitude of the incentive $\icv(t)$: if it is too small, the agent might not choose the arm $a_t$ proposed by \princip\ whereas using an overly large amount leads the principal to overpay, decreasing her utility. 

This trade-off also arises in dynamic pricing, where sellers must strike a balance between attractive pricing and profitability. For discussion of the results in that literature, see the comprehensive overview by \citet{den2015dynamic}. In addition, there are links between dynamic pricing and bandit problems~\citep[see, e.g.,][]{javanmard2019dynamic, cai2023doubly}.
\paragraph{Optimal incentives.} Before introducing $\algU$, we highlight a pivotal observation. For any given round $t \geq 1$, action $a \in [K]$ and $\varepsilon > 0$, the principal can entice the agent to choose $a$ by offering an incentive, $\icvstare_a \in \rset_+$, defined as:
\begin{equation}
\label{equation:definition_optimal_incentives}
\icvstare_a = \max_{a' \in [K]} \s_{a'} - \s_a +\epsilon \eqsp.
\end{equation}
With this incentive, it holds that for any $a'\in [K], a' \ne a$:
\begin{equation*}
\s_{a'} < \s_a + \icvstare_{a} \eqsp,
\end{equation*}
which ensures that the agent chooses $A_t = a$, given that action $a$ yields a superior reward. Consequently, $\icvstar_a \coloneqq \lim_{\epsilon \to 0} \icvstare_a = \max_{a' \in [K]} \s_{a'} - \s_a$ represents the infimal incentive necessary to make arm $a$ the agent's selection. Assuming  $\s$ is known to the principal, then using $\icvstare_a$ for any $\epsilon > 0$ across all arms $a \in [K]$, will provide an expected reward of $\theta_a - \icvstar_a$ per arm, which can be found using a standard bandit algorithm. \Cref{lemma:regretdefinitionnoncontextual} allows us to define the regret in a more convenient way.

\begin{restatable}{lemma}{regretdefinitionnoncontextual}
\label{lemma:regretdefinitionnoncontextual}
    For any $T \in \N$, the regret of any algorithm on our problem instance can be written as
    \begin{align*}
    \regret(T) &= T \max_{a \in [K]} \{ \theta_a + \s_a - \max_{a' \in [K]} \s_{a'}\} \\
    &\phantom{=}- \E\parentheseDeux{\sum_{t=1}^T \{\theta_{A_t} - \indi{a_t}(A_t)\icv(t)\}} \eqsp.
\end{align*}
\end{restatable}

\textbf{Warm up: fixed horizon solution and regret analysis.} $\algU$ separates the problem of learning optimal incentives $\icvstar_a$ for each action $a$---a problem that can be solved efficiently via binary search (see \Cref{algorithm:binary_search})---from estimation of the principal's expected reward $(\theta_a)_{a\in [K]}$, which is achieved using a standard multi-armed bandit algorithm. With a known horizon $T$, the algorithm unfolds in two stages. First, for each action $a \in [K]$, \princip\ devotes $\bsrounds \coloneqq \lceil\log_2 T\rceil$ rounds of binary search per arm to  estimate $\icvstar_a$, maintaining lower $\licv_a(\bsrounds)$ and upper $\uicv_a(\bsrounds)$ bounds with $\licv_a(\bsrounds) \leq \icvstar_a \leq \uicv_a(\bsrounds)$. Denoting $\uicv_a \coloneqq \uicv_a(\lceil \log_2 T \rceil)$ for simplicity, we compute the estimate 
\begin{equation}
\label{equation:definition_incentive_estimate}
\hicv_a \coloneqq \uicv_a(\lceil \log_2 T \rceil) + 1/T \eqsp,
\end{equation} where  $1/T$ is added to avoid any tie-breaking situation. We show formally in \Cref{lemma:precision_incentives} that 
\begin{equation}
\textstyle
\label{eq:close_incentives}
\hicv_a - 2/T  \leq \icvstar_a < \hicv_a \eqsp.
\end{equation}

In the second phase, $\algU$ then employs an arbitrary multi-armed bandit subroutine $\Alg$ in a black-box manner to learn $\theta$.

\textbf{Bandit instance.} For any  distributions $(\tnu_a)_{a\in [K]}$ and sequence of \iid\  random variables $(Y_a(t))_{t \in \N^\star}$, $a \in [K]$, with  $Y_a(t) \sim \tnu_a$, we define the history $\cG_{t} \coloneqq (A_s, U_s, Y_{A_s}(s))_{s \leq t}$ where for any $s\in \N^\star$, $(U_s)_{s\in \N^\star}$ is a family of independent uniform random variables on $[0,1]$ allowing for randomization in the subroutine. Let $\Alg$ be a bandit algorithm, i.e., $\Alg \colon (U_t, \cG_{t-1}) \mapsto \Arec_t$. We define the expected regret of $\Alg$ as
\begin{align*}
    R_{\Alg}(T, \tnu) &\coloneqq T\max_{a \in [K]} \E_{Y \sim \tnu}\parentheseDeux{Y_a(1)
    } \\
    &\phantom{\coloneqq}- \E_{Y \sim\tnu}\left[\sum_{t =1}^T Y_{\Alg(U_t, \cG_{t-1})}(t)\right] \eqsp.
\end{align*}
After the binary search phase, for $t>K\lceil \log_2 T \rceil$, the principal \ plays $\Alg$  on her bandit instance driven by her own mean rewards $(\theta_a)_{a\in [K]}$ and the approximated incentives $(\hicv_a)_{a \in [K]}$. $\Alg$ will be fed with a shifted history, defined for any $t>K\lceil \log_2 T \rceil$ as 
\begin{equation}
\label{alg_history_definition}
\begin{gathered}
    \alghist_{t} \coloneqq (\Arec_s, U_s, X_{A_s}(s)-\hicv_{\Arec_s})_{s\in [K\lceil \log_2 T \rceil+1, t]} \eqsp,
\end{gathered}
\end{equation}
with $\Arec_t$ the action recommended by $\Alg$ at time $t$ and $A_t$ the action pulled by the agent. At time $t$, $\algU$ offers the incentive $\hicv_{\Arec_t}$ to the agent if he chooses action $\Arec_t$. \Cref{eq:close_incentives} ensures that this incentive makes $\Arec_t$ strictly preferable to any other action for the agent and so $\Arec_t$ is eventually played. As can be seen in \eqref{alg_history_definition}, the shift of each arm's mean by $\hicv_a$ is taken into account while $\Alg$ is learning. We also define the shifted distribution $\shiftr_a$ for any $a \in [K]$ as the distribution of $X_a(1) - \hicv_a$.

\begin{algorithm}[!ht]
\caption{$\algU$}\label{algorithm:iterations}
\begin{algorithmic}[1]
    \STATE {\bfseries Input:} Set of actions $\cA \coloneqq [K]$, time horizon $T$,  subroutine $\Alg$
    \STATE Compute $\alghist_{s} \coloneqq \varnothing$ for any $s \leq K\lceil \log_2 T \rceil$
    \FOR{$a \in [K]$}
        \STATE \textcolor{blue}{\# See \Cref{algorithm:binary_search}}
        \STATE $\licv_a, \uicv_a =$ \texttt{Binary Search}($a, \lceil \log_2 T\rceil, 0,1$) 
    \ENDFOR
    \STATE For any action $a \in [K]$, $\hicv_{a}=\uicv_a + 1/T$
    \FOR{$t = \lceil \log_2 T\rceil+1,\ldots,T$}
        \STATE Sample $U_t \sim \Unif(0,1)$ and get a recommended action by $\Alg$, $\Arec_{t} = \Alg(U_t, \alghist_{t -1})$
        \STATE \textbf{Offer an incentive} $\hicv_{\Arec_t}$ on action $\Arec_t$ and nothing for any other action $a' \in [K], a' \ne a$
        \STATE \textbf{Observe} $A_t, X_{A_t}(t)$ and compute history \begin{small}$\alghist_{t} = (\Arec_s, U_s, X_{A_s}(s)-\hicv_{\Arec_s})_{s\in [K \lceil \log_2 T \rceil + 1,t]} $\end{small}
    \ENDFOR
\end{algorithmic}
\end{algorithm}

\begin{theorem}\label{theorem:regret_bound}
$\algU$ run with any multi-armed bandit subroutine $\Alg$ has an overall regret $\regret(T)$ such that
\begin{align*}
\regret(T) &\leq 2 + (1+\max_{a \in [K]} \{\theta_a\} - \min_{a \in [K]} \{\theta_a\})(1+ K \log_2 T) \\
& \phantom{\leq}+ R_{\Alg}(T-K\lceil\log_2 T\rceil, \shiftr) \eqsp,
\end{align*}
where $R_{\Alg}$ stands for the regret induced by $\Alg$ on the shifted vanilla multi-armed bandit problem $\shiftr$.
\end{theorem}
The proof is postponed to \Cref{app:proof_non_contextual}.

\begin{corollary}
\label{corollary:UCB_non_contextual}
Assume the principal's reward distribution $\nu_a$ for any action $a\in [K]$ is 1-subgaussian. Then, $\algU$ run with the bandit subroutine $\Alg =\texttt{UCB}$ has a regret bounded for any $T\in\N$ as follows:
\begin{align*}
 \regret(T) & \leq 3 + 3\sum_{a \in [K], \Delta^\star_a>0}\Delta^\star_a  \\
& \phantom{\leq} +(1+\max_{a \in [K]} \{\theta_a\} - \min_{a \in [K]} \{\theta_a\}) (1+ 9 K \log_2 T) \\
& \phantom{\leq}+ 8 \min \left\{\sqrt{T K \log T} \; ; \; \sum_{a\in [K], \dels_a > 0} \frac{4 \log T}{\dels_a} \right\} \eqsp,
\end{align*}
where $\dels_a \coloneqq \max_{a' \in [K]} \{ \theta_{a'} + \s_{a'} \} - (\theta_a + \s_a)$ are the reward gaps.
\end{corollary}

Note that any black-box algorithm, not necessarily $\texttt{UCB}$, can be employed, yielding other concrete bounds in the corollary. We recover the usual multi-armed $\texttt{UCB}$ bounds (both distribution-free and instance-dependent): this is why $\algU$ achieves the bound provided in \Cref{corollary:UCB_non_contextual}. For completeness, the \texttt{UCB subroutine} is given in \Cref{appendix:algorithms}.
\section{Contextual Principal-Agent Learning}\label{sec:contextual}

In this section, we study the same interaction between a principal and an agent, but in a contextual setting \citep[see, e.g.,][]{abe1999associative, auer2002using, dani2008stochastic}. We use the simplified model of stochastic linear bandits for both the agent and the principal. Consider a set of possible actions in $\oball(0,1)$, where $\oball(0,1)$ stands for the unit closed ball in $\R^d$, and a family of zero-mean distributions indexed by $\oball(0,1)$, $(\tilde{\nu}_a)_{a \in \oball(0,1)}$ such that for any $a \in \oball(0,1), t \in [T], \eta_a(t) \sim \tilde{\nu}_a$. The principal's reward is given by the sequence $\{(X_a(t))_{a\in \oball(0,1)}\colon t \in [T]\}$ of independent random variables such that for any $t\in [T], a \in \oball(0,1)$,
\begin{equation*}
    X_a(t) \coloneqq \langle \theta^\star, a\rangle + \eta_a(t) \eqsp,
\end{equation*}
where $\theta^\star$ is unknown to the principal. The agent's reward function is defined as $a \mapsto \langle \s^\star, a\rangle$, where $\s^\star \in \R^d$ is known to the agent and unknown to the principal. With this notation, the agent and the principal observe an action set $\cA_t \subseteq \oball(0,1)$, at each round $t\geq 1$.  Note that this set is no longer stationary. The precise timeline is as follows. At each round, the principal proposes an incentive function, $\tsfr(t, \cdot) \colon \cA_t \to \R_+$, associating any action $a \in \cA_t \subseteq \oball(0,1)$ with a transfer of incentives $\tsfr(t,a)$ from the principal to the agent. The principal chooses $\tsfr(t, \cdot)$ as a function with a finite support, which makes it upper semi-continuous. The agent then greedily chooses the action $A_t$ as follows:
\begin{equation}
\label{eq:contextual_action}
    A_t \in \argmax_{a \in \cA_t} \{ \langle \s^\star, a \rangle + \tsfr(t,a) \} \eqsp,
\end{equation}
which is well-defined since $\tsfr(t, \cdot)$ is upper semi-continuous and $\cA_t$ satisfies the following assumption.
\begin{assumption}
\label{assumption:bounded_in_unit_ball}
    For any $t\geq 1$, $\cA_t$ is closed, therefore compact. Moreover, $\s^\star \in \oball(0,1)$ and $\theta^\star \in \oball(0,1)$.
\end{assumption}
The principal then observes the arm $A_t$ selected by the agent, as well as her incurred reward given by $X_{A_t}(t)$. The utility of the principal on the round is given by $X_{A_t}(t) - \tsfr(t,A_t)$. This defines, for a sequence of principal's incentive functions $\{ \tsfr(t, \cdot), t\in [T]\}$, the sequence of actions $\{A_t\colon t\in [T]\}$ chosen by the agent. The goal of the principal is to maximize her total utility. On a single round $t$, she thus aims at proposing an optimal incentive function $\tsfr(t, \cdot)$ which solves
\begin{equation}
\label{eq:def_objective_princple_0_contextual}
\begin{gathered}
 \textstyle \text{maximize} \,\,  \langle \theta^\star, a \rangle -\tsfr(t,a) \text{ over } \tsfr(t, \cdot) \colon \cA_t\to \R_+, \\
 \text{such that } \, a \in \argmax_{a' \in \cA_t}\{ \langle \s^\star, a'\rangle + \tsfr(t,a')\} \eqsp.
\end{gathered}
\end{equation}
In addition, we define the optimal average reward at $t$ as
\begin{equation}
\label{equation:definition_mu_t}
\begin{gathered}
\mu_t^\star \coloneqq \sup_{\tsfr(t, \cdot) \colon \cA_t \to \R_+} \{\langle \theta^\star, a\rangle - \tsfr(t,a)\} \\
\text{ such that } a \in \argmax_{a' \in \cA_t}\{ \langle \s^\star, a' \rangle + \tsfr(t,a') \} \eqsp.
\end{gathered}
\end{equation}
Maximizing the total utility of the principal over $T$ rounds is equivalent to minimizing the expected regret, defined as
\begin{align}
\label{contextual_regret_def}
\small    \regret(T) &= \sum_{t=1}^T \mu_t^\star - \E\parentheseDeux{\sum_{t=1}^T (X_{A_t}(t)-\tsfr(t,A_t))} \eqsp.
\end{align}
Similarly to \Cref{lemma:regretdefinitionnoncontextual}, the following result provides an alternative definition for the regret.

\begin{restatable}{lemma}{regretdefinitioncontextual}
\label{lemma:regretdefinitioncontextual}
    For any $T \in \N$, the regret of any algorithm on our contextual problem instance can be written as
    \begin{align*}
    \regret(T) &=\sum_{t=1}^T \max_{a \in \cA_t} \{ \langle \theta^\star + \s^\star, a\rangle - \max_{a' \in \cA_t} \langle \s^\star, a'\rangle \} \\
    &\phantom{=}- \E\parentheseDeux{\sum_{t=1}^T \left(\langle \theta^\star, A_t\rangle-  \tsfr(t,A_t)\right)} \eqsp.
    \end{align*}
\end{restatable}
The proof is deferred to \Cref{app:contextual_setting}.

\textbf{Design of the optimal incentives.} At any round $t\geq 1$, for the agent to necessarily choose action $a \in \cA_t$, the principal can provide the agent with the incentive function $\tsfrstare_a(t,a') \coloneqq \indi{a}(a')\icvstare(t,a)$, where for any $\epsilon >0$,
\begin{equation*}
    \icvstare(t,a) \coloneqq \max_{a_t' \in \cA_t}\{  \langle \s^\star, a_t' -a\rangle +\epsilon \} \eqsp.
\end{equation*}
\Cref{lemma:good_action_chosen_contextual} in \Cref{app:contextual_setting} guarantees that this choice of $\tsfrstare_a$ gives $A_t=a$, where $A_t$ is defined in \eqref{eq:contextual_action}. Define 
\begin{gather}
\nonumber
    a_t^\agent \coloneqq \argmax_{a' \in \cA_t} \langle \s^\star, a' \rangle \eqsp,  \\
 \label{equation:optimal_incentives_contextual}
    \icvstar(t,a) \coloneqq \langle \s^\star, a_t^\agent \rangle - \langle \s^\star, a\rangle
\end{gather}
and $\tsfrstar_a(t,a') \coloneqq \indi{a}(a')\icvstar(t,a)$. As in the non-contextual case, taking $\epsilon \to 0$ makes the incentive function $\tsfrstar_a$ the infimal function that makes the choice of $a$ strictly preferable to any other arm $a' \in \cA_t$ at time $t$.

Similarly to the multi-armed setting, we decompose the problem into two distinct components. First, we aim to estimate the agent's reward $a \mapsto \ps{\s^\star}{a}$ based on the observation of agent's selected actions given an appropriate choice of incentives. As discussed below, this can be achieved with a binary-search-like procedure. Second, once this function is accurately estimated, the principal can use a contextual bandit algorithm $\CAlg$ in a black-box manner to minimize her own regret with the estimated incentive function to determine the agent's behavior.

\textbf{Estimation of the agent's reward.} 
The approach that we propose is based on a sequence of confidence sets $\{\cS_t \}_{t \in [T]}$ that satisfy $\s^\star \in \cS_t$ for any $t \in [T]$. We construct the sequence $(\cS_t)_{t \in [T]}$ recursively such that their diameters decrease along the iterations. This is motivated by \Cref{lemma:controlvolume} which allows us to control the estimation error of $\icvstar$ and relates it to the diameter of these sets. The proof is postponed to \Cref{app:contextual_setting}.
\begin{restatable}{lemma}{controlvolume}
    \label{lemma:controlvolume}
   For any $t \in [T]$ and closed subset $\cS \subset \ball{0}{1}$ with $\s^{\star} \in \cS$, it holds, for any $a \in \cA_t$, $|\max_{\s \in \cS,a' \in \cA_t}\langle \s, a' - a \rangle - \icvstar(t,a)| \leq 2 \, \diam(\cS, \cA_t)$ where $\diam(\cS, \cA_t) \coloneqq \max_{a' \in \cA_t} \max_{\s_1, \s_2 \in \cS} |\langle \s_1 - \s_2, a' \rangle |$.
\end{restatable}
In the light of \Cref{lemma:controlvolume}, we thus aim to build confidence sets $\cS_t$ with decreasing diameters such that $\s^{\star}\in\cS_t$ for any~$t$. To this end, the principal can offer an incentive function~$\tsfr(t,\cdot)$ concentrated on a \textit{single point} $a \in \cA_t$ as in the multi-armed case: $\tsfr(t,a') = \icv(t) \cdot \indi{a}(a')$ for $\icv(t) \in \R_+$. In this case, the principal receives the agent's choice as a feedback; by \eqref{eq:contextual_action}, either $A_t = a_t$ or $A_t = \argmax_{a' \in \cA_t} \langle \s^\star, a' \rangle = a_t^\agent$. In addition, $A_t=a_t$ is equivalent to the fact that
\begin{equation*}
    \langle \s^\star, a_t^\agent - a_t \rangle \leq \icv(t) \eqsp.
\end{equation*}
The information $A_t = a_t$ or $A_t = a_t^\agent$ can be used as \textit{binary search feedback} in the direction $a^{\agent}_t - a_t$, as follows. Given a current confidence set $\cS_t$ at time $t$ it can be updated either as $\cS_{t+1} = \cS_t \cap \{\s\colon \langle \s, a_t^\agent - a_t \rangle \leq \icv(t) \}$ if $A_t = a_t$ or $\cS_{t+1} = \cS_t \cap \{\s \colon \langle \s, a_t^\agent - a_t \rangle \geq \icv(t) \}$ otherwise.

However, since the action set $\cA_t$ is non-stationary, we cannot determine the action $a^{\agent}_t$ by just observing the first round. Consequently, the new set $\cS_{t+1}$ cannot be computed as previously in the non-contextual setting. This makes the \textit{single-point} incentive functions not suited for an efficient learning of $\s^{\star}$ over the iterations. Instead, at any time $t$, we seek for a form of binary-search feedback in the direction $a^1_t - a^t_2$ for any two arms $a^1_t \not= a_t^2 \in \cA_t$. As we will see, this can be achieved by considering an incentive function $\tsfr$ with support $\{a_t^1, a_t^2\} \subseteq \cA_t$.

Indeed, an important remark is that the amount of incentive needed to make the agent play any particular action is bounded under \Cref{assumption:bounded_in_unit_ball} since
\begin{equation}
\label{equation:bounded_incentives_contextual}
    \max_{a\in \cA_t} \icvstar(t,a) = \max_{a \in \cA_t} \langle \s^*, a_t^\agent - a\rangle \leq 2 \eqsp,
\end{equation}
this bound being known by the principal. For any $a \in\oball(0,1)$, this makes the incentive function $a' \mapsto 3 \cdot \indi{a}(a')$ sufficient to ensure $A_t = a$ from \eqref{eq:contextual_action}. The value $3$ in the definition of the incentive function is chosen instead of $2$ to avoid an arbitrary tie-breaking.

Consequently, under the choice $\tsfr(t,a_t^1)=3$, $\tsfr(t,a_t^2) = 3 + \icv(t)$ for $\icv(t) \geq 0$  and $\tsfr(t, a') = 0 $ for any other arm $a'$, \eqref{equation:bounded_incentives_contextual} guarantees that only $a_t^1$ and $a_t^2$ may be chosen by the agent, helping the principal to update her confidence set $\cS_t$ in a known direction. Specifically for such an incentive $\tsfr$, the choice $A_t = a_t^1$ reveals the following information on $\s^\star$:
\begin{equation*}
\langle \s^\star, a_t^1 - a_t^2 \rangle \geq \icv(t),
\end{equation*}
that permits the definition of a \textit{binary search-like feedback} in the direction $a_t^1-a_t^2$ and thus allows us to update the confidence set $\cS_t$ following $\cS_{t+1} = \cS_t \cap \{\s\colon \langle \s, a_t^1 - a_t^2 \rangle \geq \icv(t) \}$ if $A_t = a_t^1$ or $\cS_{t+1} = \cS_t \cap \{\s \colon \langle \s, a_t^1 - a_t^2 \rangle \leq \icv(t) \}$ otherwise.

\textbf{Binary search.} This update turns our estimation of the optimal incentives $\icvstar$ into a multidimensional binary search where the unknown quantity is the vector $\s^\star$. At each iteration $t$, a vector $w_t$ from the unit sphere is given. Then, the algorithm has to guess the value of $\langle \s^\star, w_t \rangle$ using its previous observations. Finally, an oracle reveals as feedback whether the guess is above or below the true value $\langle \s^\star, w_t\rangle$, and the algorithm updates its observation history. In our case, $w_t \coloneqq (a_t^1-a_t^2)/\|a_t^1-a_t^2\|$ for $a_t^1, a_t^2 \in \cA_t$ and the resulting feedback is given through the agent picking either $a_t^1$ or $a_t^2$. However, extending the binary search to the multidimensional case is non-trivial for two reasons.

\textbf{Direction of the multidimensional binary search.} In the contextual bandit setting, we cannot divide the horizon into two successive phases. Indeed, the principal cannot choose any binary search direction in $\R^d$, since $w_t$ depends on the action set $\cA_t$ available at each iteration. For instance, action sets $\cA_t$ could be restricted to a small dimensional subspace of $\R^d$ during the whole binary search procedure, so that the principal can only get a good estimate of $\s^\star$ in this subspace. After this phase, received action sets could be totally different (e.g., in the orthogonal subspace or the whole of $\R^d$) during the remainder of the game. 

We solve the issue of constraint directions for the binary search by running it in an adaptive way, depending on the available action set at each time step and on the current level of estimation on this set. More precisely, at iteration $t$, the principal's estimate of the true value $\langle \s^\star, w_t \rangle$ is $\langle \hat{\s}_t, w_t \rangle$, where $\hat{\s}_t$ is defined as the centroid of $\cS_t$:
\begin{equation*}
    \hat{\s}_t \coloneqq \frac{1}{\Vol(\cS_t)} \int_{\cS_t} x \rmd x \quad\text{with} \quad \Vol(\cS_t) = \int_{\cS_t} \rmd x \eqsp.
\end{equation*}
Whenever $|\langle \hat{\s}_t, w_t \rangle - \langle \s^\star, w_t \rangle | < 1/T$, the principal incurs a negligible cost to incentivize the agent to choose her desired action. Then, in this context, for any action $a \in \cA_t$ that the principal wants to play, she designs the incentive
\begin{align*}
    \hicv(t,a) &\coloneqq \max_{a' \in \cA_t} \langle \hat{\s}_t, a'\rangle - \langle \hat{\s}_t, a\rangle + 2/T \\
    \estsfr_a(t,a') &= \indi{a}(a')\hicv(t,a) \quad \text{for any }a'\in\cA_t\eqsp.
\end{align*}
To control the precision of the estimation $\hicv(t,a)$ of $\icvstar(t,a)$ for any $a \in \cA_t$,  \Cref{lemma:enough_incentives_contextual} shows that it is sufficient to consider  the event $\cE_t$, defined   as
\begin{equation}
\label{equation:definition_cE_t}
\cE_t \coloneqq \left\{ \max_{a_t^1\ne a_t^2 \in \cA_t} \diam \left(\cS_t, \frac{a_t^1-a_t^2}{\|a_t^1-a_t^2\|}\right) < \precision \right\} \eqsp,
\end{equation}
where we recall the definition of the projected diameter: $\diam(\cS_t, x) \coloneqq \max_{\s_1, \s_2 \in \cS_t} |\langle \s_1 - \s_2, x\rangle | \; \text{ for any } x\in \R^d$. When $\cE_t$ is false, the principal does not have a good characterization of the incentive function that she needs to provide and thus $\calgU$ runs a multidimensional binary search step, which is explained in the paragraph below. Otherwise, $\calgU$ runs a contextual bandit subroutine $\CAlg$ in a black-box manner on her bandit instance driven by the principal's own mean rewards $\langle \theta^\star, a\rangle$ and the approximated incentives $\hicv(t,a)$ for any $a \in \cA_t$. \Cref{lemma:enough_incentives_contextual} guarantees that these approximations are upper estimates of $\icvstar(t,a)$. The principal proposes an incentive function $\estsfr_{\Arec_t}$ depending on the estimate to make the agent select the action $\Arec_t$ recommended by the bandit subroutine. Again, we do not impose any assumption on the tie-breaking, which can be arbitrary.
\begin{restatable}{lemma}{enoughincentivescontextual}
\label{lemma:enough_incentives_contextual}
    Consider $t\in [T]$,  $\cA_t \subseteq \oball(0,1)$, $\cS_t \subseteq \oball(0,1)$ such that $\cE_t$ defined in \eqref{equation:definition_cE_t} is true. Then for any action $a \in \cA_t$, we have: $\icvstar(t,a) < \estsfr_a(t,a) \leq \icvstar(t,a) + 4/T$.
\end{restatable}

A corollary of \Cref{lemma:enough_incentives_contextual} is that running $\calgU$, under $\cE_t$, $A_t = \Arec_t$.

\textbf{Issue of the diameter reduction.} We illustrate the challenge of the multidimensional constrained binary search on a very simple problem. At time $t$, we can only run a binary search step in one of the directions $w_t=a - a'$ for $a, a' \in \cA_t$. Suppose that we have two directions of interest, $v_1, v_2$ in $\R^d$, such that we aim to decrease the diameter of $\cS_t$ in the direction of $v_1$ or $v_2$. Even if we divide the diameter of $\cS_t$ in two in a direction $w_t$, which is always possible, this does not necessarily imply that the diameter of $\cS_t$ would reduce along any direction $v_i$, as illustrated on \Cref{figure:binary_search_is_difficult}.

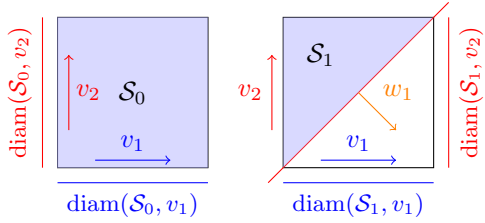
\begin{figure}[ht]
\centering
\begin{tikzpicture}
  % Draw square
  \draw (0,0) -- (2,0) -- (2,2) -- (0,2) -- cycle;
  \fill[blue!30,opacity=0.5] (0,0) -- (2,0) -- (2,2) -- (0,2) -- cycle;
    \node[align=center] at (1,1) {$\cS_0$};
  \draw[blue] (0,-0.2) -- (2,-0.2);
  \node[align=center, text=blue] at (1,-0.5) {\begin{small}$\diam(\cS_0,v_1)$\end{small}};
  \draw[red] (-0.2,0) -- (-0.2,2);
  \node[rotate=90, text=red] at (-0.5,1) {\begin{small}$\diam(\cS_0,v_2)$\end{small}};
  \draw (3,0) -- (5,0) -- (5,2) -- (3,2) -- cycle;
  % Draw diagonal for the second square
  \draw[red] (2.8,-0.2) -- (5.2,2.2);
  % Add x-axis label for the second square
  \draw[blue] (3,-0.2) -- (5,-0.2);
  \node[align=center, text=blue] at (4,-0.5) {\begin{small}$\diam(\cS_1,v_1)$\end{small}};
  \node[above, text=blue] at (1,0.1) {$v_1$};
  \node[left, text=red] at (0.7,1) {$v_2$};
  \node[left, text=red] at (2.85,1) {$v_2$};
  \draw[->,blue] (0.5, 0.1) -- (1.5, 0.1);
  \draw[->,blue] (3.5, 0.1) -- (4.5, 0.1);
  \draw[->,red] (0.15, 0.5) -- (0.15, 1.5);
  \draw[->,red] (2.85, 0.5) -- (2.85, 1.5);
  \draw[->,orange] (4, 1) -- (4.5, 0.5);
  \node[right, text=orange] at (4.2,1) {$w_1$};
  \draw[red] (5.2,0) -- (5.2,2);
  \node[rotate=90, text=red] at (5.5,1) {\begin{small}$\diam(\cS_1,v_2)$\end{small}};
  \fill[blue!30,opacity=0.5] (3,0) -- (3,2) -- (5,2) -- cycle;
  \node[align=center] at (3.5,1.5) {$\cS_1$};
  \node[above, text=blue] at (4,0.1) {$v_1$};
\end{tikzpicture}
\caption{\label{figure:binary_search_is_difficult}Illustration of a case where the volume $\cS_0$ is cut along a direction $w_1$ to give a new confidence set $\cS_1$; while the diameter is not reduced along the directions $v_1$ nor $v_2$.
}
\end{figure}

An early attempt to tackle this multidimensional binary search problem with adversarial directions was presented by \citet{cohen2020feature}, who used ellipsoid methods. Here, we use the recent strategy proposed by \citet{lobel2018multidimensional} and their $\pvaU$ subroutine, which is described further in \Cref{appendix:algorithms}.

\textbf{Non-stationarity of the reward shift.} For any rewards $\{(Y_a(t))_{a\in \cA_t}\colon t \in [T]\}$, we define the history as $\cG_t \coloneqq (\cA_s, A_s, U_s, Y_{A_s}(s))_{s\leq t}$ where $(U_s)_{s \in \N}$ is a family of independent uniform random variables on $[0,1]$ to allow randomization in the decision making. Let $\CAlg$ be a linear contextual bandit algorithm, i.e., $\CAlg \colon (U_t, \cG_{t-1}, \cA_t) \mapsto \Arec_t\in\cA_t$. When the principal is not running a binary search step, i.e., when $\cE_t$ is true, she plays the $\CAlg$ subroutine on her bandit instance. We define a subset $I_t$ of all the iterations during which $\CAlg$ is run and a shifted history $\alghist_t$ available at time $t$ as 
\begin{align}
    \label{equation:definition_I_t}
    & I_t \coloneqq \{s \in [t] \text{ such that } \cE_s \text{ is true}\} \eqsp, \\
    \nonumber
    \text{ and } \; & \alghist_t \coloneqq (\cA_s, A_s, U_s, X_{\Arec_s}(s) -\estsfr_{\Arec_t}(s,A_s) )_{s\in I_t}.
\end{align}
In our setup, $\CAlg$ will be fed in a black-box manner with this shifted history to issue recommendations $\Arec_t$, $\CAlg$ $\colon (U_t, \alghist_{t-1}, \cA_t) \mapsto \Arec_t$. At time $t$, for an action $\Arec_t$ recommended by $\CAlg$, our meta-algorithm $\calgU$ proposes an incentive designed so that the agent eventually picks $\Arec_t$ (\Cref{lemma:enough_incentives_contextual}): $A_t = \Arec_t$.

However, the last difference between the non-contextual and contextual cases is that the shift between $\hicv(t,a_t)$ and $\icvstar(t,a_t)$ is not constant anymore on bandit steps $t \in I_T$. % since the binary search is run on an ongoing basis up to time $T$.
This shift of rewards is interpreted as adversarial corruption \citep{bubeck2012best,lykouris2018stochastic}.

At each round, taking into account this shift, the optimal average utility associated with action $a \in \cA_t$ for the principal is $r^\star(t,a) := \langle \theta^\star, a \rangle - \icvstar(t,a)$, while the principal can only estimate a \textit{non-stationary} expected reward\footnote{Even if we were to feed the stochastic observations $(X_{A_s}(s)-\hat{\pi}(t,A_s))_{s\leq t}$ at time $t$, past algorithmic decisions would depend on different observation distributions, making the direct use of classical regret bounds of the bandit subroutine impossible.} $r^\star(t,a) + \epscorruption_t$ with the \textit{corruption level} $\epscorruption_t$ defined as
\begin{equation}
\label{equation:definition_espcorruption}
\epscorruption_t := \icvstar(t,A_t) - \hicv(t,A_t) \,
\end{equation}
and $\epscorruption_{I_t} \coloneqq (\epscorruption_s)_{s \in I_t}$. In this setup, we can define a corrupted regret as follows
\begin{multline}
    \label{equation:Rcorrupt_definition}
    R^{\text{corrupt}}_{\CAlg}(I_T, \epscorruption_{I_T}) \\= %\sup_{ |\epscorruption_s| \leq \corruptiont\ \forall s\in [T]} 
    \E\parentheseDeux{\sum_{t \in I_T} \max_{a \in \cA_t} r^\star(t,a) - r^\star(t, \Arec_t)} \eqsp,
\end{multline}
where $\Arec_t = \CAlg(U_t, \alghist_{t-1}(\epscorruption_{I_{t-1}}), \cA_t)$ and $\alghist_{t}(\epscorruption_{I_t}) = (\cA_s, A_s, U_s, r^\star(s,\Arec_s) + \eta_{\Arec_s}(s) + \epscorruption_s)_{s \in I_t}$. Then, we aim to minimize the corrupted regret with $\CAlg$, which is not possible using a naive linear contextual bandit algorithm.

\begin{algorithm}[!ht]
\caption{$\calgU$}\label{algorithm:contextual_algorithm}
\begin{algorithmic}[1]
    \STATE {\bfseries Input:} horizon $T$, subroutine $\CAlg$, $\pvacst = 1/16T^2 d (d+1)^2$
    \STATE {\bfseries Initialize:} $\alghist_0 = V_0 = \varnothing$, $\cS_0 = \{s\in \mathbb{R}^d:\|s\|\leq 1\}$ 
    \FOR{$t = 1, \ldots, T$}
        \STATE \textbf{Observe} available action set $\cA_t$
        \IF{$\cE_t$ is \texttt{FALSE}} \STATE \textcolor{blue}{\# Where $\cE_t$ is defined in \eqref{equation:definition_cE_t}}
        %\IF{$\exists a_t^1 \ne a_t^2 \in \cA_t\text{: diam}(\cS_t, \frac{a_t^1-a_t^2}{\|a_t^1-a_t^2\|})\geq 1/T$}
            \STATE $\cS_{t+1}, \VV_{t+1}$
            \STATE $= \pvaU(T, \pvacst, \cS_t, \VV_t, a_t^1, a_t^2)$
        \ELSE{}
            \STATE Compute $\hat{\s}_t$ as the centroid of $\cS_t$
            \STATE Sample $U_t \sim  \Unif(0,1)$
            \STATE Get $\Arec_{t} = \CAlg(U_t, \alghist_t, \cA_t)$
            \STATE Let $\hicv(t, \Arec_t) = \max\limits_{a' \in \cA_t} \langle \hat{\s}_t, a'\rangle - \langle \hat{\s}_t, \Arec_t \rangle +\frac{2}{T}$
            \STATE \textbf{Propose} incentive function $\estsfr_{\Arec_t}(t,a) = \indi{\Arec_t}(a) \cdot \hicv(t,\Arec_t)$
            \STATE \textbf{Observe} $A_t$ as defined in \eqref{eq:contextual_action}, $X_{A_t}$
            \STATE Update $\alghist_t$ with $ (\cA_t, A_t, U_t, X_{A_t} - \tsfr(t,A_t))$
        \ENDIF
    \ENDFOR
\end{algorithmic}
\end{algorithm}
\textbf{Regret analysis.} We split the regret into three components, each of them being bounded separately. One of these components comes from the bias in the estimation of the optimal incentives.% to make the agent choose some of the actions. 
%Therefore, the principal always incentivizes a bit more to make sure the agent chooses the desired action.
Secondly, the principal incurs a cost due to the iterations of $\CAlg$ on the corrupted bandit instance. We use the results from \citet{he2022nearly} with a known corruption level to bound this term.
Finally, the last term follows from the multidimensional binary search steps used to estimate $\s^\star$. \Cref{lemma:boundedE} allows us to bound the number of such steps; see \Cref{app:contextual_setting} for a proof which builds on the work of  \citet{lobel2018multidimensional}.
\begin{restatable}{lemma}{llboundedE}
\label{lemma:boundedE}
Consider $\cE_t$ defined by \eqref{equation:definition_cE_t} with $(\cS_t)_{t \in [T]}$ defined by $\calgU$. Then it holds almost surely that
\begin{equation*}
    \sum_{t\geq1} \1_{\cE_t^\complementary}
    \leq 192 \cdot d\log (dT) \eqsp,
\end{equation*}
where $\cE_t$ is defined by \eqref{equation:definition_cE_t}.
\end{restatable}

\begin{restatable}{theorem}{regc}
\label{theorem:regret_bound_contextual}
If $\calgU$ is run with any linear contextual subroutine $\CAlg$, then with the same constant $192$ as in \Cref{lemma:boundedE}, the regret of $\calgU$ is bounded as
\begin{equation*}
    \regret(T) \leq 2 + R_\CAlg^{\text{corrupt}}(I_T, \epscorruption_{I_T}) + 1344 \, d \log dT \eqsp.
\end{equation*}
\end{restatable}

We emphasize that our results still hold with any contextual linear bandit algorithm and that the overall regret is mostly driven by the term $R_\CAlg^{\text{corrupt}}$: although the principal has to solve simultaneously a \textit{pricing-like} problem and a \textit{stochastic bandit} problem, she almost achieves linear bandit state-of-the art regret.

The main difference between traditional and corrupted rewards in the bandit setting lies in the fact that, in the former case, rewards are typically assumed to be i.i.d., whereas in the latter case, they may be chosen adversarially. An algorithm is robust to corruptions if it yields regret guarantees for any possible reward corruption within a specific budget. This kind of problem was first considered by \citet{bubeck2012best} and has been extensively studied since \citep[see, e.g.,][]{kapoor2019corruption}. In our setting, a bound for the corruption budget is available, thanks to \Cref{lemma:shifted_reward_bound} below. 
\begin{restatable}{lemma}{shiftedrewardbound}
\label{lemma:shifted_reward_bound}
    Consider $(\cE_t)_{t \in [T]}$ defined by \eqref{equation:definition_cE_t} with $(\cS_t)_{t \in [T]}$ defined by $\calgU$. Let $I_T$ and $(\epscorruption_t)_{t \in [T]}$ as defined in \eqref{equation:definition_I_t} and \eqref{equation:definition_espcorruption} and $t \in [T]$ such that $\cE_t$ is true. Then
    $|\epscorruption_t| \leq 4/T \; \text{ and }\; \sum_{t \in I_T} |\epscorruption_t| \coloneqq \corruption \leq 4$.
\end{restatable}
 With standard bandit assumptions, we can then consider a corruption robust algorithms, such as $\Corralg$ from \citet{he2022nearly}.
\begin{assumption}
\label{assumption:contextual_subgaussian}
    At each round $t \geq 1$, for any action $a \in \cA_t$, the principal's reward $X_a(t)$ is $\alghist_t$-conditionally 1-subgaussian, i.e., for any $\lambda \in \R$, we have $\E\parentheseDeux{\rme^{\lambda (X_a(t) - \E\parentheseDeux{X_a(t)})}|\cH_{t-1}} \leq \rme^{\lambda^2/2}$.
\end{assumption}

\begin{corollary}\label{corollary:regret_bound_contextual}
Suppose that \Cref{assumption:contextual_subgaussian} is true. If \Cref{algorithm:contextual_algorithm} is run with the subroutine $\CAlg \coloneqq \Corralg$ proposed by \citet{he2022nearly}, the regularization parameter $\lambda=1$ and a confidence level $\probreg = 1/T$, the following bound holds
\begin{align*}
    \regret(T) \leq 11 + 1344 \, d \log (dT) + C_{\CAlg} d \sqrt{T} \log T ,
\end{align*}
with $C_{\CAlg}$ being an universal constant.
\end{corollary}
As in the multi-armed setting, the obtained regret bounds are comparable to the achievable best performance in the standard bandit settings, where the principal does not need to estimate the agent's parameters $s^{\star}$.

\section{Experiments}

\looseness=-1 We illustrate our theoretical findings with experiments on a toy example and compare $\algU$ with the \texttt{Principal's $\varepsilon$-Greedy algorithm} of \citet{dogan2023repeated}. We also compare with a \texttt{UCB Oracle} baseline that runs \texttt{UCB} on the shifted bandit instance with arm means $\mu_a \coloneqq \theta_a - \icvstar_a$ and no principal-agent consideration. This baseline corresponds to the case where the principal knows the agent's reward vector and therefore only has to consider a bandit algorithm. Experimental details can be found in \Cref{appendix:experiments}. We observe in \Cref{figure:regret_comparison_non_contextual} that the Principal $\epsilon$-Greedy Algorithm from \citet{dogan2023repeated} exhibits suboptimal performance. Additionally, another issue arises from its computational complexity, requiring an optimization step at every round. In comparison, $\algU$ yields a regret nearly equal to the one of \texttt{Oracle UCB}, illustrating that the cost of estimating the agent's preferences, obtained from binary search, is negligible for $\algU$.
\begin{figure}[!htb]\centering
    \includegraphics[width=0.42\textwidth,trim=1cm 0cm 2cm 1.5cm, clip]{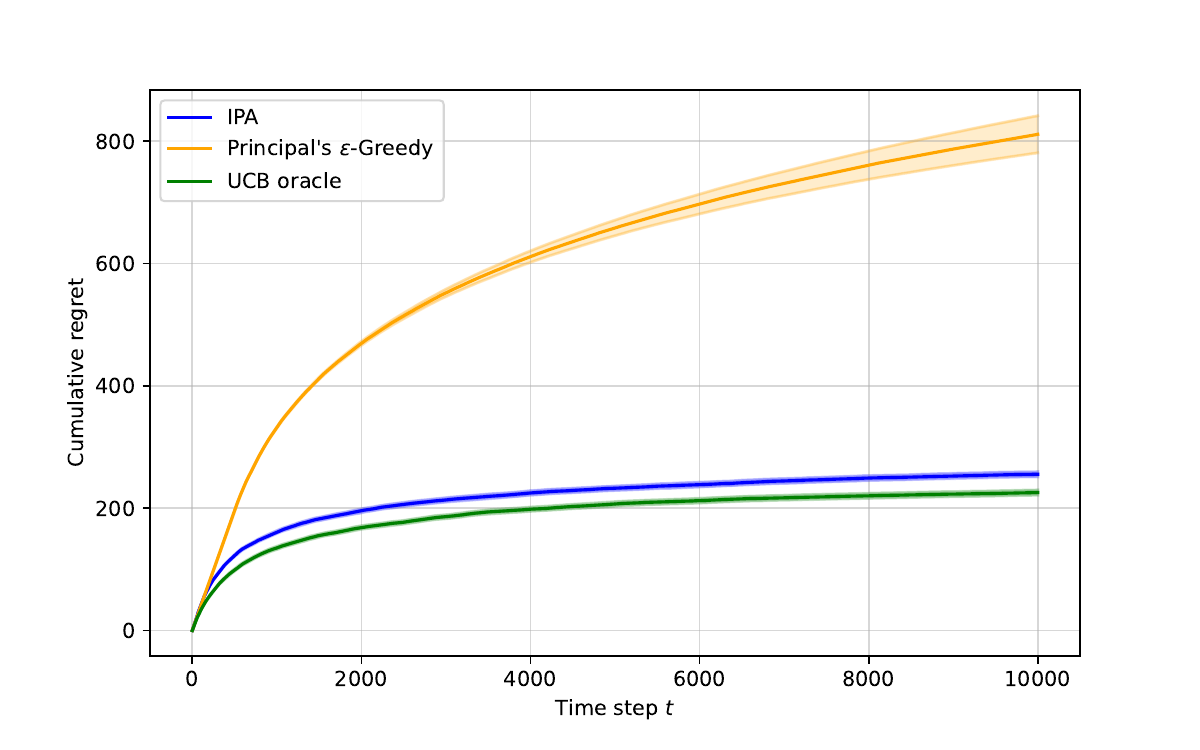}
    \caption{Cumulative regret for different algorithms on a $5$ arms instance.}
    \label{figure:regret_comparison_non_contextual}
\end{figure}

\section{Lower Bounds} For the sake of clarity, we stick to the multi-armed case of \Cref{sec:IPA} in this section. A simple observation yields that  
\begin{equation*}
    \regret(T) \geq \E\parentheseDeux{\sum_{t=1}^T \mu^\star - \mu_{A_t} }\eqsp, \label{eq:lowerbound}
\end{equation*}
where $\mu_a = \theta_a - \pi^{\star}_a$ and $\mu^{\star}=\max_{a\in[K]}\mu_a$. 
Even if the principal was to know the optimal incentives $(\icvstar_a)_{a \in [K]}$, she would still face a bandit instance with arm means $\mu_a$. From there, we can directly extend standard lower bounds from the bandit literature to our setting \citep{lai1985asymptotically,burnetas1996optimal}.
\begin{proposition}
\label{proposition:lower_bound}
Let $\cD$ be a class of distributions. Consider the multi-armed case of \Cref{sec:IPA} and a policy satisfying for any instance $\nu \in \cD^K$ and $\alpha>0$, $\regret(T) = o(T^{\alpha})$. Then, for any $\nu\in\cD^K$,
\begin{equation*}
\liminf_{T\to\infty} \frac{\regret(T)}{\log T}\geq \sum_{a,\mu_a<\mu^{\star}} \frac{\mu^{\star}-\mu_a}{\mathrm{KL}_{\inf}(\nu_a-\pi^{\star}_a,\mu^{\star},\cD)} \eqsp,
\end{equation*}
where denoting by $\mathrm{KL}$ the Kullback-Leibler divergence,
\begin{multline*}
        \mathrm{KL}_{\inf}(\rho,\mu^{\star},\cD)\\
    \coloneqq \!\inf\limits \{\mathrm{KL}(\rho,\rho') \colon\! \rho' \in \cD\eqsp, \int x \rho'(\rmd x) > \mu^\star \} \eqsp,
\end{multline*}
\end{proposition}
\vspace{-0.4cm}
The complete proof is postponed to \Cref{appendix:lower_bound}. \Cref{proposition:lower_bound} states that $\algU$ yields a nearly optimal regret. Similar arguments can be made in the contextual setting.
\section{Conclusion and Possible Extensions}

This paper presents two novel algorithms called $\algU$ and $\calgU$, tackling generalizations of both multi-armed and contextual bandits that account for principal-agent interactions.
By decoupling the learning of the agent and the estimation of the principal's parameters, we are able to obtain a nearly optimal algorithm, improving over the previous work of \citet{dogan2023repeated}. 
Overall, we obtain an efficient principal-agent bandit framework that allows us to take into account an interaction between a principal and an agent with misaligned interests in a bandit environment. There are various possible extensions of our work, among which considering strategic behaviors for repeated interactions with a single agent or uncertainty on the agent's side. %Such extensions are discussed in \Cref{appendix:discussion}.

%\clearpage
%\pagebreak
% \section*{Impact Statements}

% This paper presents work whose goal is to advance the field of Machine Learning. There are many potential societal consequences of our work, none which we feel must be specifically highlighted here.

%\section{Discussion}
%\label{appendix:discussion}

\textbf{Information rent.} Again, we consider the multi-armed case for the sake of clarity. We assumed that the agent is always greedy and therefore chooses at time $t$ an action following
\begin{equation}
\label{equation:action_extension}
    A_t = \argmax_{a \in \cA_t} \{ s_a + \indi{a_t}(a) \icv(t) \} \eqsp.
\end{equation}
However, nothing prevents the agent from lying and choosing another action instead of $A_t$. The maximal total welfare that can be extracted at each round is $\max_{a \in \cA} \{ s_a + \theta_a\}$. In our setting, with a trustful agent, this reward was shared between the two actors with an average reward $\max_{a' \in \cA} s_{a'}$ for the agent and $\max_{a \in \cA} \{s_a + \theta_a\} - \max_{a' \in \cA} s_{a'}$ for the principal. However, as it is exposed by \citet[][Section 4]{dogan2023repeated}, the agent could play with a malicious policy and choose $A_t$ as if he had different $(s_a)_{a \in [K]}$. In that case, he can extract an individual reward $\max_{a \in \cA} \{s_a + \theta_a\} - \min_{a' \in \cA} \theta_{a'}$, while letting a reward $\min_{a' \in [K]} \theta_{a'}$ to the principal. In that case, the agent exloits his \textit{information rent} to increase his profits. Against such adversarial and powerful agents, the principal cannot do more than play and learn with the $(s_a)_{a \in [K]}$ \textit{announced} by the agent. However, this situation is not an issue with myopic agents who act greedily since each of them tries to maximize his own instantaneous reward and eventually select $A_t$ from \eqref{equation:action_extension}. This situation is encountered in many applications, where each agent has a single round interaction with the principal for the whole game.

\textbf{Learning agents.} A possible extension would be to incorporate uncertainty on the agent's side and consider learning agents \citep[see, e.g.,][]{dogan2023estimating}. However again, when considering single round interactions with the agents, each agent myopically maximizes his reward \textit{a priori}. Consequently, the agent policy is stationary and would be driven by $s_a$, the expected \textit{beliefs} on the action rewards. The single interaction model is already well suited for numerous real world applications. In the case of repeated interactions between the principal and a single learning agent,  it becomes much more complex, as this agent can both learn his true rewards on the run while trying to influence future actions of the principal with his own choices. Restricting the agent's policy to a specific set might then be necessary, as done by \citet{dogan2023estimating} with Agent's $\epsilon$-Greedy strategy. The major learning difficulty from the principal side would then come from the non-stationarity of the agents decisions and could be handled using non-stationary bandits algorithms \citep[see, e.g.,][]{gittins1979bandit, lattimore2020bandit}.

\section*{Acknowledgements}
Funded by the European Union (ERC, Ocean, 101071601). Views and opinions expressed are however those of the author(s) only and do not necessarily reflect those of the European Union or the European Research Council Executive Agency. Neither the European Union nor the granting authority can be held responsible for them. The work of DT has been supported by the Paris Île-de-France Région in the framework of DIM AI4IDF.

\bibliography{sample}

\begin{thebibliography}{44}
\providecommand{\natexlab}[1]{#1}
\providecommand{\url}[1]{\texttt{#1}}
\expandafter\ifx\csname urlstyle\endcsname\relax
  \providecommand{\doi}[1]{doi: #1}\else
  \providecommand{\doi}{doi: \begingroup \urlstyle{rm}\Url}\fi

\bibitem[Abe \& Long(1999)Abe and Long]{abe1999associative}
Abe, N. and Long, P.~M.
\newblock Associative reinforcement learning using linear probabilistic concepts.
\newblock In \emph{ICML}, pp.\  3--11, 1999.

\bibitem[Auer(2002)]{auer2002using}
Auer, P.
\newblock Using confidence bounds for exploitation-exploration trade-offs.
\newblock \emph{Journal of Machine Learning Research}, 3\penalty0 (Nov):\penalty0 397--422, 2002.

\bibitem[Banihashem et~al.(2023)Banihashem, Hajiaghayi, Shin, and Slivkins]{banihashem2023bandit}
Banihashem, K., Hajiaghayi, M., Shin, S., and Slivkins, A.
\newblock Bandit social learning: Exploration under myopic behavior.
\newblock \emph{arXiv preprint arXiv:2302.07425}, 2023.

\bibitem[Ben-Porat et~al.(2023)Ben-Porat, Mansour, Moshkovitz, and Taitler]{ben2023principal}
Ben-Porat, O., Mansour, Y., Moshkovitz, M., and Taitler, B.
\newblock Principal-agent reward shaping in mdps.
\newblock \emph{arXiv preprint arXiv:2401.00298}, 2023.

\bibitem[Bertsimas \& Vempala(2004)Bertsimas and Vempala]{bertsimas2004solving}
Bertsimas, D. and Vempala, S.
\newblock Solving convex programs by random walks.
\newblock \emph{Journal of the ACM (JACM)}, 51\penalty0 (4):\penalty0 540--556, 2004.

\bibitem[Boursier \& Perchet(2022)Boursier and Perchet]{boursier2022survey}
Boursier, E. and Perchet, V.
\newblock A survey on multi-player bandits.
\newblock \emph{arXiv preprint arXiv:2211.16275}, 2022.

\bibitem[Bubeck \& Slivkins(2012)Bubeck and Slivkins]{bubeck2012best}
Bubeck, S. and Slivkins, A.
\newblock The best of both worlds: Stochastic and adversarial bandits.
\newblock In \emph{Conference on Learning Theory}, pp.\  42--1. JMLR Workshop and Conference Proceedings, 2012.

\bibitem[Burnetas \& Katehakis(1996)Burnetas and Katehakis]{burnetas1996optimal}
Burnetas, A.~N. and Katehakis, M.~N.
\newblock Optimal adaptive policies for sequential allocation problems.
\newblock \emph{Advances in Applied Mathematics}, 17\penalty0 (2):\penalty0 122--142, 1996.

\bibitem[Cai et~al.(2023)Cai, Chen, Wainwright, and Zhao]{cai2023doubly}
Cai, J., Chen, R., Wainwright, M.~J., and Zhao, L.
\newblock Doubly high-dimensional contextual bandits: An interpretable model for joint assortment-pricing.
\newblock \emph{arXiv preprint arXiv:2309.08634}, 2023.

\bibitem[Castiglioni et~al.(2020)Castiglioni, Celli, Marchesi, and Gatti]{castiglioni2020online}
Castiglioni, M., Celli, A., Marchesi, A., and Gatti, N.
\newblock Online bayesian persuasion.
\newblock \emph{Advances in Neural Information Processing Systems}, 33:\penalty0 16188--16198, 2020.

\bibitem[Cohen et~al.(2022)Cohen, Deligkas, and Koren]{cohen2022learning}
Cohen, A., Deligkas, A., and Koren, M.
\newblock Learning approximately optimal contracts.
\newblock In \emph{International Symposium on Algorithmic Game Theory}, pp.\  331--346. Springer, 2022.

\bibitem[Cohen et~al.(2020)Cohen, Lobel, and Paes~Leme]{cohen2020feature}
Cohen, M.~C., Lobel, I., and Paes~Leme, R.
\newblock Feature-based dynamic pricing.
\newblock \emph{Management Science}, 66\penalty0 (11):\penalty0 4921--4943, 2020.

\bibitem[Conitzer \& Garera(2006)Conitzer and Garera]{conitzer2006learning}
Conitzer, V. and Garera, N.
\newblock Learning algorithms for online principal-agent problems (and selling goods online).
\newblock In \emph{Proceedings of the 23rd international conference on Machine learning}, pp.\  209--216, 2006.

\bibitem[Dani et~al.(2008)Dani, Hayes, and Kakade]{dani2008stochastic}
Dani, V., Hayes, T.~P., and Kakade, S.~M.
\newblock Stochastic linear optimization under bandit feedback.
\newblock 2008.

\bibitem[Den~Boer(2015)]{den2015dynamic}
Den~Boer, A.~V.
\newblock Dynamic pricing and learning: historical origins, current research, and new directions.
\newblock \emph{Surveys in operations research and management science}, 20\penalty0 (1):\penalty0 1--18, 2015.

\bibitem[Dogan et~al.(2023{\natexlab{a}})Dogan, Shen, and Aswani]{dogan2023estimating}
Dogan, I., Shen, Z.-J.~M., and Aswani, A.
\newblock Estimating and incentivizing imperfect-knowledge agents with hidden rewards.
\newblock \emph{arXiv preprint arXiv:2308.06717}, 2023{\natexlab{a}}.

\bibitem[Dogan et~al.(2023{\natexlab{b}})Dogan, Shen, and Aswani]{dogan2023repeated}
Dogan, I., Shen, Z.-J.~M., and Aswani, A.
\newblock Repeated principal-agent games with unobserved agent rewards and perfect-knowledge agents.
\newblock \emph{arXiv preprint arXiv:2304.07407}, 2023{\natexlab{b}}.

\bibitem[Evans et~al.(2017)Evans, Terhorst, and Kang]{evans2017data}
Evans, K.~J., Terhorst, A., and Kang, B.~H.
\newblock From data to decisions: helping crop producers build their actionable knowledge.
\newblock \emph{Critical reviews in plant sciences}, 36\penalty0 (2):\penalty0 71--88, 2017.

\bibitem[Gittins(1979)]{gittins1979bandit}
Gittins, J.~C.
\newblock Bandit processes and dynamic allocation indices.
\newblock \emph{Journal of the Royal Statistical Society Series B: Statistical Methodology}, 41\penalty0 (2):\penalty0 148--164, 1979.

\bibitem[Golrezaei et~al.(2023)Golrezaei, Jaillet, and Liang]{golrezaei2023incentive}
Golrezaei, N., Jaillet, P., and Liang, J. C.~N.
\newblock Incentive-aware contextual pricing with non-parametric market noise.
\newblock In \emph{International Conference on Artificial Intelligence and Statistics}, pp.\  9331--9361. PMLR, 2023.

\bibitem[Gr{\"o}tschel et~al.(2012)Gr{\"o}tschel, Lov{\'a}sz, and Schrijver]{grotschel2012geometric}
Gr{\"o}tschel, M., Lov{\'a}sz, L., and Schrijver, A.
\newblock \emph{Geometric algorithms and combinatorial optimization}, volume~2.
\newblock Springer Science \& Business Media, 2012.

\bibitem[He et~al.(2022)He, Zhou, Zhang, and Gu]{he2022nearly}
He, J., Zhou, D., Zhang, T., and Gu, Q.
\newblock Nearly optimal algorithms for linear contextual bandits with adversarial corruptions.
\newblock \emph{Advances in Neural Information Processing Systems}, 35:\penalty0 34614--34625, 2022.

\bibitem[Hu et~al.(2022)Hu, Ngo, Slivkins, and Wu]{hu2022incentivizing}
Hu, X., Ngo, D., Slivkins, A., and Wu, S.~Z.
\newblock Incentivizing combinatorial bandit exploration.
\newblock \emph{Advances in Neural Information Processing Systems}, 35:\penalty0 37173--37183, 2022.

\bibitem[Javanmard \& Nazerzadeh(2019)Javanmard and Nazerzadeh]{javanmard2019dynamic}
Javanmard, A. and Nazerzadeh, H.
\newblock Dynamic pricing in high-dimensions.
\newblock \emph{The Journal of Machine Learning Research}, 20\penalty0 (1):\penalty0 315--363, 2019.

\bibitem[Kamenica \& Gentzkow(2011)Kamenica and Gentzkow]{kamenica2011bayesian}
Kamenica, E. and Gentzkow, M.
\newblock Bayesian persuasion.
\newblock \emph{American Economic Review}, 101\penalty0 (6):\penalty0 2590--2615, 2011.

\bibitem[Kapoor et~al.(2019)Kapoor, Patel, and Kar]{kapoor2019corruption}
Kapoor, S., Patel, K.~K., and Kar, P.
\newblock Corruption-tolerant bandit learning.
\newblock \emph{Machine Learning}, 108\penalty0 (4):\penalty0 687--715, 2019.

\bibitem[Laffont \& Martimort(2009)Laffont and Martimort]{laffont2009theory}
Laffont, J.-J. and Martimort, D.
\newblock The theory of incentives: the principal-agent model.
\newblock In \emph{The Theory of Incentives}. Princeton University Press, 2009.

\bibitem[Lai \& Robbins(1985)Lai and Robbins]{lai1985asymptotically}
Lai, T.~L. and Robbins, H.
\newblock Asymptotically efficient adaptive allocation rules.
\newblock \emph{Advances in Applied Mathematics}, 6\penalty0 (1):\penalty0 4--22, 1985.

\bibitem[Lattimore \& Szepesv{\'a}ri(2020)Lattimore and Szepesv{\'a}ri]{lattimore2020bandit}
Lattimore, T. and Szepesv{\'a}ri, C.
\newblock \emph{Bandit Algorithms}.
\newblock Cambridge University Press, 2020.

\bibitem[Li et~al.(2010)Li, Chu, Langford, and Schapire]{li2010contextual}
Li, L., Chu, W., Langford, J., and Schapire, R.~E.
\newblock A contextual-bandit approach to personalized news article recommendation.
\newblock In \emph{Proceedings of the 19th International Conference on the World Wide Web}, pp.\  661--670, 2010.

\bibitem[Lobel et~al.(2018)Lobel, Leme, and Vladu]{lobel2018multidimensional}
Lobel, I., Leme, R.~P., and Vladu, A.
\newblock Multidimensional binary search for contextual decision-making.
\newblock \emph{Operations Research}, 66\penalty0 (5):\penalty0 1346--1361, 2018.

\bibitem[Lykouris et~al.(2018)Lykouris, Mirrokni, and Paes~Leme]{lykouris2018stochastic}
Lykouris, T., Mirrokni, V., and Paes~Leme, R.
\newblock Stochastic bandits robust to adversarial corruptions.
\newblock In \emph{Proceedings of the 50th Annual ACM SIGACT Symposium on Theory of Computing}, pp.\  114--122, 2018.

\bibitem[Mansour et~al.(2020)Mansour, Slivkins, and Syrgkanis]{mansour2020bayesian}
Mansour, Y., Slivkins, A., and Syrgkanis, V.
\newblock Bayesian incentive-compatible bandit exploration.
\newblock \emph{Operations Research}, 68\penalty0 (4):\penalty0 1132--1161, 2020.

\bibitem[Mao et~al.(2018)Mao, Leme, and Schneider]{mao2018contextual}
Mao, J., Leme, R., and Schneider, J.
\newblock Contextual pricing for lipschitz buyers.
\newblock \emph{Advances in Neural Information Processing Systems}, 31, 2018.

\bibitem[Myerson(1989)]{myerson1989mechanism}
Myerson, R.~B.
\newblock \emph{Mechanism design}.
\newblock Springer, 1989.

\bibitem[Rademacher(2007)]{rademacher2007approximating}
Rademacher, L.~A.
\newblock Approximating the centroid is hard.
\newblock In \emph{Proceedings of the Twenty-Third Annual Symposium on Computational Geometry}, pp.\  302--305, 2007.

\bibitem[Rusmevichientong \& Tsitsiklis(2010)Rusmevichientong and Tsitsiklis]{rusmevichientong2010linearly}
Rusmevichientong, P. and Tsitsiklis, J.~N.
\newblock Linearly parameterized bandits.
\newblock \emph{Mathematics of Operations Research}, 35\penalty0 (2):\penalty0 395--411, 2010.

\bibitem[Sellke \& Slivkins(2021)Sellke and Slivkins]{sellke2021price}
Sellke, M. and Slivkins, A.
\newblock The price of incentivizing exploration: A characterization via thompson sampling and sample complexity.
\newblock In \emph{Proceedings of the 22nd ACM Conference on Economics and Computation}, pp.\  795--796, 2021.

\bibitem[Simchowitz \& Slivkins(2023)Simchowitz and Slivkins]{simchowitz2023exploration}
Simchowitz, M. and Slivkins, A.
\newblock Exploration and incentives in reinforcement learning.
\newblock \emph{Operations Research}, 2023.

\bibitem[Slivkins et~al.(2019)]{slivkins2019introduction}
Slivkins, A. et~al.
\newblock Introduction to multi-armed bandits.
\newblock \emph{Foundations and Trends{\textregistered} in Machine Learning}, 12\penalty0 (1-2):\penalty0 1--286, 2019.

\bibitem[Smith \& Vamanamurthy(1989)Smith and Vamanamurthy]{smith1989small}
Smith, D.~J. and Vamanamurthy, M.~K.
\newblock How small is a unit ball?
\newblock \emph{Mathematics Magazine}, 62\penalty0 (2):\penalty0 101--107, 1989.

\bibitem[Thompson(1933)]{thompson1933likelihood}
Thompson, W.~R.
\newblock On the likelihood that one unknown probability exceeds another in view of the evidence of two samples.
\newblock \emph{Biometrika}, 25\penalty0 (3-4):\penalty0 285--294, 1933.

\bibitem[Woodroofe(1979)]{woodroofe1979one}
Woodroofe, M.
\newblock A one-armed bandit problem with a concomitant variable.
\newblock \emph{Journal of the American Statistical Association}, 74\penalty0 (368):\penalty0 799--806, 1979.

\bibitem[Yu et~al.(2021)Yu, Liu, Nemati, and Yin]{yu2021reinforcement}
Yu, C., Liu, J., Nemati, S., and Yin, G.
\newblock Reinforcement learning in healthcare: A survey.
\newblock \emph{ACM Computing Surveys (CSUR)}, 55\penalty0 (1):\penalty0 1--36, 2021.

\end{thebibliography}
\bibliographystyle{icml2024}

%%%%%%%%%%%%%%%%%%%%%%%%%%%%%%%%%%%%%%%%%%%%%%%%%%%%%%%%%%%%%%%%%%%%%%%%%%%%%%%
%%%%%%%%%%%%%%%%%%%%%%%%%%%%%%%%%%%%%%%%%%%%%%%%%%%%%%%%%%%%%%%%%%%%%%%%%%%%%%%
% APPENDIX
%%%%%%%%%%%%%%%%%%%%%%%%%%%%%%%%%%%%%%%%%%%%%%%%%%%%%%%%%%%%%%%%%%%%%%%%%%%%%%%
%%%%%%%%%%%%%%%%%%%%%%%%%%%%%%%%%%%%%%%%%%%%%%%%%%%%%%%%%%%%%%%%%%%%%%%%%%%%%%%
\newpage
\appendix
\onecolumn
%\appendix

\section{Notation}
\label{appendix:notations}

\begin{table}[!ht]
\centering
\setlength\tabcolsep{3.5pt}
\begin{tabular}{ |c|p{10cm}| }
  \hline $\cA \coloneqq [K]$ & {Set of possible arms.} \\
  \hline $T$ & {Horizon.} \\
  \hline $N_T \coloneqq \log_2 T$ & {Number of steps dedicated to the binary search on each arm in $\algU$.} \\
  \hline $a_t$ & {Arm on which the principal offers an incentive.} \\
  \hline
  $\icv(t)$ & {Amount of incentive offered by the principal on action $a_t$.} \\ \hline
  $A_t$ & {Arm chosen by the agent, maximizing his utility, known by everyone.} \\ \hline
  $\s_a + \indi{a_t}(a)\icv(t)$ & {Agent's utility for action $a$.} \\ \hline
  $\nu_a$ & {Principal's reward distribution for action $a$.} \\ \hline
  $X_{a}(t) - \indi{a_t}(a)\icv(t)$ & {Principal's utility for action $a$.} \\ \hline
  $\mu_a$ & {Principal's expected utility for action $a$, using the optimal incentive $\icvstar_a$.} \\ \hline
  $\mu^\star$ & {Maximal expected utility for the principal.} \\ \hline
  $\theta_a \coloneqq \E\parentheseDeux{X_a(1)}$ & {Principal's expected reward.} \\ \hline
  %$\tsfr(t,a)$ & {Transfer from the principal to the agent.} \\ \hline
  $\icvstar_a$ & {Infimum amount of incentives to be offered on action $a_t=a$ to make the agent choose it.} \\ 
  \hline $\alghist$ & {Shifted history used to feed $\Alg$.} \\
  \hline $R_{\Alg}(T)$ & {Regret of the subroutine $\Alg$ on a horizon $T$.} \\
  \hline $\regret(T)$ & {Overall regret of $\algU$ on a horizon $T$.} \\ \hline
\end{tabular}
\caption{Notations used in \Cref{sec:IPA}.}
\label{table:notations}
\end{table}

\begin{table}[!ht]
\centering
\setlength\tabcolsep{3.5pt}
\begin{tabular}{ |c|p{13cm}| }
  \hline
  $T$& {Horizon.} \\ \hline
  $\oball(0,1)$& {Unit ball in $\R^d$, $d\geq 1$.} \\ \hline
  $\cA_t \subseteq \oball(0,1)$ & {Action set at time $t$ among which the agent selects $A_t$.} \\ \hline
  $\icv(t)$ & {Amount of incentive offered by the principal on some action.} \\ \hline
  $\eta_a(t)$ & {Noise distribution of the principal's reward associated with action $a$ at time $t$.} \\ \hline
  $r^\star(t,a) = \langle \theta^\star, a\rangle + \eta_a(t)- \icvstar(t,a)$ & {Utility collected by the principal on action $a$ at time $t$ if the optimal amount of incentive is used.} \\ \hline
  $\mu_t^\star$ & {Principal's maximal expected utility at time $t$.} \\ \hline
  $\icvstar(t,a)$ & {Infimal amount of incentives to be offered on action $a$ with $\tsfr(t,a') = \1_{a}(a')\icvstar_a$ so that the agent eventually chooses action $a$.} \\ \hline
  $\hicv(t,a)$ & {Principal's estimation of $\icvstar(,a)$} \\ \hline
  $\tsfr(t, \cdot) \colon \oball(0,1) \to \R_+$& {Incentive function, associating each action with some amount of incentives.} \\ \hline
  $\s^\star \in \oball(0,1)$ & {Agent's true reward vector.} \\ \hline
  $\cS_t \ni \s^\star$ & {Principal's confidence set for $\s^\star$ at time $t$.} \\ \hline
  $\langle \s^\star, a\rangle + \tsfr(t,a)$ & {Agent's utility for action $a$.} \\ \hline
  $a_t^\agent = \argmax_{a \in \cA_t} \langle \s^\star, a\rangle$ & {Agent's optimal action with null incentives at time $t$.} \\ \hline
  $\Arec_t$ & {Action recommended by $\CAlg$ at time $t$.} \\ \hline
  $\langle \theta^\star, a \rangle + \eta_a(t) - \tsfr(t,a)$ & {Principal's utility for action $a$.} \\ \hline
  $\mu^\star_t$ & {Maximal expected utility for the principal at time $t$.} \\ \hline
  $\epscorruption_t = \icvstar(t,a) - \hicv(t,a)$ & {Shift between the optimal incentives and the estimated ones.} \\ \hline
  $\corruption$ & {Total corruption budget due to the shift between $\hicv(t,a)$ and $\icvstar(t,a)$ over the rounds.} \\ \hline
  $\cE_t$ & {Event being true if the diameter of $\cS_t$ projected $\cA_t$ is small than $1/T$.} \\ \hline
  $I_t$ & {Rounds up to time $t$ during which $\cE_s, s \leq t$ is true.} \\
  \hline $\alghist$ & {Shifted history used to feed $\Alg$.} \\
  \hline $R_{\CAlg}(T)$ & {Regret of the subroutine $\CAlg$ on a horizon $T$.} \\
  \hline $\regret(T)$ & {Overall regret of $\calgU$ on a horizon $T$.} \\ \hline
\end{tabular}
\caption{Notation used in \Cref{sec:contextual}.}
%\vspace{-7mm}
\label{table:notations_contextual}
\end{table}
\section{Experimental details}
\label{appendix:experiments}

We ran the experiments in \Cref{figure:regret_comparison_non_contextual} for a horizon $T= 10\, 000$ on an average of $100$ runs on a five arms bandit. We plotted the standard error across the different runs. The expected rewards for the principal ($\theta$) and the agent ($\s$) are given in \Cref{table:experimental_data}. The principal's rewards $X_a(t)$ are drawn from an i.i.d. distribution $X_a(t) \sim \cN(\theta_a, 1)$ for any $a \in [K], t \in [T]$. We also run an oracle \texttt{UCB} instance with rewards following a Gaussian distribution $\cN(\mu_a, 1)$ where for any $a \in [K]$, $\mu_a \coloneqq \theta_{a} + \s_a - \max_{a' \in [K]} \s_{a'}$, as if a \texttt{UCB} algorithm was run with the full knowledge of the optimal incentives and was learning his own mean rewards ($\mu$), taking into account these incentives. The mean rewards $\mu$ are also given in \Cref{table:experimental_data}. We observe that the additional exploration steps needed to learn the optimal incentives in $\algU$ are not very costly compared to the regret achieved by the \texttt{UCB} oracle.

For the Principal's $\epsilon$-Greedy algorithm, we use the hyperparameters $\alpha =1$ and $m=500$. The hyperparameter $m$ controls the number of exploration steps. We ran the \texttt{Principal's $\epsilon$-Greedy} algorihtm on the same bandit setting for different values $m=30, m=100, m=200, m=300, m=400, m=500, m=600, m=800, m=1000, m=2000, m=5000, m=10\,000$. Below $m=500$, the algorithm does not explore enough and incurs a linear regret on some runs, consequenly yielding a poor mean regret, whereas above $m=500$, the algorithm explores excessively, leading to a higher regret due to overexploration. We ran the same experiments on longer horizons $T=100\,000$ and $T=1 \, 000 \, 000$ and the algorithms exhibited the same behavior. In practice, the tuning of the \texttt{$\epsilon$-Greedy} algorithm depends on the reward gaps and is not common to use. This is why another advantage of $\algU$ compared to the \texttt{Principal's $\epsilon$-Greedy algorithm} of \citet[][]{dogan2023repeated} lies in the fact that it does not need any tuning of hyperparameters, leading to a better use in practice, on potentially broader bandit instances.

\begin{table}[!ht]
\centering
\setlength\tabcolsep{5pt}
\begin{tabular}{ |@{\hspace{10pt}}c@{\hspace{10pt}}||ccccc| }
  \hline $\s$ &   $0.64$ &  $0.99$ & $0.73$ & $0.61$ & $0.59$ \\
  \hline $\theta$ & $0.30$ & $0.24$ & $0.88$ &  $0.07$ & $0.65$ \\
  \hline $\mu$ & $-0.05$ & $0.24$ & $0.62$ &  $-0.31$ & $0.25$ \\\hline
\end{tabular}
\caption{\label{table:experimental_data}Experimental parameters for \Cref{figure:regret_comparison_non_contextual}.}
% \vspace{-7mm}
\end{table}
We did not run $\calgU$ in a contextual bandit setting because it is quite tedious to implement, due to the use of the \texttt{Projected Volume} subroutine from the work of \citet[][]{lobel2018multidimensional}. Even though they obtain an excellent regret bound, the computations raise specific challenges. The first issue is the computation of the centroid which is known to be a $\#$P-hard problem \citep{rademacher2007approximating}. However, it can be solved through an approximation of the centroid, which is computable in polynomial time \citep[see,][Lemma 5 and Theorem 12]{bertsimas2004solving}. A second issue is finding directions along which the set $\cS_t$ has a small diameter, which is needed to compute the set $V_t$. It is solved by \citet{lobel2018multidimensional} with an ellipsoidal approximation $E$ of $\cS_t$ such that $E \subseteq \cS_t \subseteq \alpha E$ with $\alpha>1$, since such an ellipsoid can be computed in polynomial time, \citep[see,][Corollary 4.6.9]{grotschel2012geometric}. Such a variation of the $\pvaU$ subroutine is presented in the work of \citet[][Section 9.3]{lobel2018multidimensional}. It is shown that one can achieve polynomial time computations with still the same regret bound for the multidimensional binary search steps \citep[][Theorem 9.4]{lobel2018multidimensional}. This line of work needs to be explored for implementing $\calgU$ in practice, which is feasible but still requires a significant amount of work.
\section{Regret Bound for Non-Contextual Setting}\label{app:proof_non_contextual}

\textbf{Notations.} We define $\uicv_a(t) \in \R_+$ as the upper estimate and $\licv_a(t) \in \R_+$ as the lower estimate of $\icvstar_a$ after $t$ rounds of binary search on arm $a$. For any $t \in [T]$ and $a \in [K]$, we define $\icvmid_a(t) \coloneqq (\uicv_a(t)+\licv_a(t))/2$. We define $\bsrounds\coloneqq \lceil \log_2 T\rceil$ as the number of binary search steps per arm and $\hicv_a$ is the estimated incentive to make the agent choose action $a$ after $\bsrounds$ steps of binary search: $\hicv_a = \uicv_a(\bsrounds)+ 1/T$. Since our problem is stationary, we write $a^\principal:=\argmax_{a\in [K]} \{ \theta_a - \icvstar_{a} \}$ for the optimal action that the principal could aim to play at each round.

\regretdefinitionnoncontextual*

\begin{proof}[Proof of \Cref{lemma:regretdefinitionnoncontextual}]
Recall that the regret is defined in \eqref{equation:regret_definition_non_contextual} as $\regret(T) \coloneqq T\, \mu^\star  - \sum_{t=1}^T \E\parentheseDeux{X_{A_t}(t) -\indi{a_t}(A_t) \icv(t) }$, where $\mu^\star = \sup_{a \in [K], \icv \in \R_+} \E_{\nu}\parentheseDeux{X_{a}(1)} - \icv \; , \text{  such that  } \; a \in \argmax_{a' \in [K]} \{ \s_{a'} + \icv\} \eqsp.$
Note that we can write $\mu^\star = \sup_{a \in [K], \icv \in \R_+} \{ \theta_a - \1_{\tilde{A}}(a, \icv) \icv \}$, where $\tilde{A} \coloneqq \{ (a, \icv) \colon \s_a + \icv \geq \max_{a'} \s_{a'} + \1_a(a')\icv\}$. First note that if $(a, \icv) \in \tilde{A}$, then $\s_a - \s_{a'} \geq - \icv \; \text{ for any }a' \in [K]$, which implies by definition of the optimal incentives \eqref{equation:definition_optimal_incentives} that $\icvstar_a  \leq \icv$. Consequently,
\begin{align*}
    \mu^\star = \max_{a \in [K]} \{ \E_{\nu}\parentheseDeux{X_{a^\optimal}(1)} - \icvstar_a\}  = \max_{a \in [K]} \{\theta_a - \max_{a' \in [K]} \{\s_{a'}\}+ \s_a \}\eqsp,
\end{align*}
hence our result about the regret.
\end{proof}

\begin{lemma}\label{lemma:bounded_incentives}
    Assume \Cref{assumption:non_contextual_bounded_reward_principal} and that we run \Cref{algorithm:binary_search} for an action $a \in [K]$ and a number of binary searches $N_T\in \N$. Then, for any $t\in [\bsrounds]$, $0 \leq \licv_a(t) \leq \icvmid_a(t) \leq \uicv_a(t)  \leq 1$.
\end{lemma}

\begin{proof}%[Proof of \Cref{lemma:bounded_incentives}]
    The proof is by induction on $t\in [N_T]$. For $t =0$, it is defined by definition. Then suppose that it holds true for $t \geq 0$. 
Note that line 6 in \Cref{algorithm:binary_search} can be written as 
\begin{equation}
\label{equation:iteration_lemma_bounded_incentives}
    \begin{aligned}
        \uicv_a(t+1) &= \indi{a}(A_t)\icvmid_a(t) + (1-\indi{a}(A_t)) \uicv_a(t) \\
        \licv_a(t+1) &=  (1-\indi{a}(A_t))\icvmid_a(t) + \indi{a}(A_t)\licv_a(t) \eqsp,
    \end{aligned}
    \end{equation}
    which completes the proof by applying the induction hypothesis.
\end{proof}

\begin{lemma}\label{lemma:precision_incentives}
    Assume \Cref{assumption:non_contextual_bounded_reward_principal} and that we run \Cref{algorithm:binary_search} for an action $a \in [K]$ and a number of binary searches $N_T\in \N$. Then, for any $t\in [\bsrounds]$,
    \begin{equation*}
    \icvstar_a \in [\licv_a(t), \uicv_a(t)] \: \text{ and } \: |\uicv_a(t) - \licv_a(t) | \leq 1/2^t \eqsp.
\end{equation*}
\end{lemma}

\begin{proof}%[Proof of \Cref{lemma:precision_incentives}]
The proof is by induction on $t$. The case $t =0$ is trivial by the initialization of \Cref{algorithm:binary_search} and \Cref{assumption:non_contextual_bounded_reward_principal}. 

Suppose that the statement holds for $t$.
Note that $ \1(\{A_{t} = a\}) = \1(\{\icvmid_a(t) \geq \icvstar_a\})$, therefore using \eqref{equation:iteration_lemma_bounded_incentives}, we obtain by using the induction hypothesis that 
\begin{align*}
    & \icvstar_a \in [\licv_a(t+1), \uicv_a(t+1)] \eqsp, \quad  \uicv_a(t+1) - \licv_a(t+1) = \frac{\uicv_a(t) - \licv_a(t)}{2} \eqsp,
\end{align*}
which completes the proof.
\end{proof}

\begin{proof}[Proof of \Cref{theorem:regret_bound}]
Recall that \Cref{lemma:regretdefinitionnoncontextual} implies that
\begin{equation}
\regret(T) = \E\left[\sum_{t=1}^T \max_{a \in [K]}\{ \theta_a - \icvstar_a\} -(X_{A_t}(t) - \indi{a_t}(A_t)\icv(t))\right]
\end{equation}
We decompose the regret between the $K \lceil \log_2 T\rceil = K \bsrounds$ first steps during which we run the \texttt{Binary Search Subroutine} and all the subsequent ones
\begin{align*}
    \regret(T) &=  \termA + \termB\\
    \termA & = \E\left[\sum_{t=1}^{K \lceil \log_2 T\rceil} \max_{a \in [K]}\{ \theta_a - \icvstar_a\} -\{X_{A_t}(t) - \indi{a_t}(A_t)\icv(t)\}\right]\\
\termB & = \E\left[\sum_{t=K \lceil \log_2 T\rceil+1}^{T} \max_{a \in [K]}\{ \theta_a - \icvstar_a\} -\{X_{A_t}(t) - \indi{a_t}(A_t)\icv(t)\}\right] \eqsp.
\end{align*}
We separate the analysis of the regret, bounding independently two terms in the right-hand side of the previous decomposition. Since $\icv(t)$ is always equal to $\icvmid_a(t)$ for some $t \in [\bsrounds], a \in [K]$, during the binary search phase, we use \Cref{lemma:bounded_incentives} to bound $\icv(t)$ by $1$ for any $t \leq K\lceil \log_2 T\rceil$ in $\termA$, giving
\begin{align}
\termA &= \E \parentheseDeux{\sum_{t=1}^{K \lceil \log_2 T\rceil} \max_{a \in [K]} \{ \theta_{a} - X_{A_t}(t) + \indi{a_t}(A_t)(t) \icv(t) -  \icvstar_{a^\principal} \}}\\
& \leq \sum_{t=1}^{K \lceil \log_2 T\rceil} (1 + \max_{a \in [K]} \{\theta_a\} - \min_{a \in [K]} \{\theta_a\}) 
\leq (1+\max_{a \in [K]} \{\theta_a\} - \min_{a \in [K]} \{\theta_a\}) K (1 + \log_2 T)\eqsp.
\end{align}
At the end of the binary search phase and for all the subsequent rounds $t> K \lceil \log_2 T\rceil$, $\Alg$ recommends an action $\Arec_t$ and the principal proposes the incentive
$\icv(t) = \hicv_{\Arec_t} = \uicv_{\Arec_t}(\lceil \log_2 T\rceil)+ 1/T$ on action $\Arec_t$ to make the agent choose it. \Cref{lemma:precision_incentives} ensures that after $\lceil \log_2 T\rceil$ rounds of binary search on action $a \in [K]$, we have
\begin{equation*}
\licv_a(\lceil \log_2 T\rceil) \leq  \icvstar_a \leq \uicv_a(\lceil \log_2 T\rceil)  \;\; \text{ and } \;\; \uicv_a(\lceil \log_2 T\rceil) - \licv_a(\lceil \log_2 T\rceil) \leq 1/2^{\lceil \log_2 T\rceil} \leq 1/T \eqsp.
\end{equation*}
Therefore,
\begin{equation*}
    \icvstar_a < \hicv_a \text{ and } \hicv_a-\icvstar_a \leq 2/T \eqsp.
\end{equation*}
For the agent, the utility associated with action $\Arec_t$ is $\s_{\Arec_t} + \hicv_{\Arec_t} >\s_{\Arec_t} + \icvstar_{\Arec_t}$, which guarantees that he eventually selects $\Arec_t$ at time $t$ because of \eqref{eq:def_A_t} and \eqref{equation:definition_optimal_incentives}. It ensures that for any $t > K\lceil \log_2 T\rceil, A_t = \Arec_t$.

To compute these recommendations, $\Alg$ is fed at any time $t\in \N$ with the shifted history defined in \eqref{alg_history_definition}: $\alghist_{t} = (\Arec_{s}, U_s, X_{\Arec_{s}}(s)-\hicv_{\Arec_{s}})_{s \in [K\lceil \log_2 T \rceil +1, t]}$. Recall that we defined the shifted distribution $\shiftr_a$ for any $a \in [K]$ as the distribution of $X_a(1) - \hicv_a$. For any $t > K \lceil \log_2 T\rceil$, $\Arec_t = \Alg(U_t, \alghist_{t-1})$ and we define $Y_a(t) \sim \shiftr_a$ for any $t \in [K \lceil \log_2 T\rceil+1, T], \, a\in [K]$. In this setup, the regret of $\Alg$ after $\tau$ subsequent steps is defined as
\begin{align*}
    R_{\Alg}(\tau, \shiftr) &= \tau \max_{a \in [K]} \E_{\shiftr}[Y_a(K\lceil \log_2 T\rceil+1)] - \E\left[\sum_{s=K \lceil \log_2 T\rceil+1}^{K\lceil \log_2 T\rceil+\tau} Y_{\Alg(U_s, \alghist_{s-1})}(s) \right] \eqsp.
\end{align*}
Consequently, since $a_t = \Arec_t = A_t$
\begin{align*}
\termB &= \E\left[\sum_{t=K \lceil \log_2 T\rceil+1}^{T} \max_{a \in [K]} \{\theta_a -  \icvstar_{a}\} -\left(X_{\Arec_t}(t)-\hicv_{\Arec_t}\right) \right] \\
& \leq \E\left[\sum_{t=K \lceil \log_2 T\rceil + 1}^T \max_{a \in [K]}\left\{ \theta_a - \hicv_a -(X_{\Arec_t}(t)-\hicv_{\Arec_t})\right\} + \max_{a'\in [K]}\left\{\hicv_{a'}- \icvstar_{a'}\right\}\right] \\
& = \E\left[\sum_{t=K \lceil \log_2 T\rceil+ 1}^T \max_{a \in [K]}\left\{ \theta_a - \hicv_a\right\} -(X_{\Arec_t}(t)-\hicv_{\Arec_t})\right] + \E\left[\sum_{t=K \lceil \log_2 T\rceil+1}^T \max_{a'\in [K]}\left\{\hicv_{a'} - \icvstar_{a'}\right\}\right] \\
& = (T-K \lceil \log_2 T\rceil)\max_{a \in [K]}\E[Y_a(K\lceil \log_2 T\rceil+1)] - \E\left[\sum_{t=K \lceil \log_2 T\rceil+ 1}^T(X_{\Arec_t}(t)-\hicv_{\Arec_t} )\right] \\
&+ (T- K \lceil \log_2 T\rceil) \max_{a' \in [K]} \{\hicv_{a'} - \icvstar_{a'}\} \\
&\leq R_{\Alg}\left(T-K \lceil \log_2 T\rceil, (\shiftr_a)_{a\in [K]}\right) + 2.
\end{align*}
Plugging $\termA$ and $\termB$ together finally gives a bound for the regret
\begin{equation*}
\regret(T) \leq 2 + (1 + \max_{a \in [K]} \{\theta_a\} - \min_{a \in [K]} \{\theta_a\})(1+K  \log_2 T)+ R_{\Alg}(T-K\lceil \log_2 T\rceil, \shiftr) \eqsp.
\end{equation*}
\end{proof}

%%%%%%%%%%%%%%%%%%%%%%%%%%%%%%%%%%%%%%%%%%%%%%%%%%%%%%%%%%%%%%%%%%%%%%%%%%%%%%%%%%%%%%%%%%%%%%%%%%%%%%%%%%%%%%%%%%%%%%%%%%%%%%%%%%%%%%%%%%%%%%%%%%%%%%%%%%%%%%
% PROOF OF THE COROLLARY

\begin{proof}[Proof of \Cref{corollary:UCB_non_contextual}]
The case $T \leq 9K$ is trivial. Assume then that $T \geq 9 K$. Note that after $\lceil \log_2 T \rceil$ rounds of binary search, $\icvstar_a \leq \uicv_a \leq \icvstar_a \, + \, 1/T$. We define $\dels_a \coloneqq \max_{a' \in [K]} \{\theta_{a'} -\icvstar_{a'}\} - (\theta_a - \icvstar_a) = \max_{a' \in [K]} \{\theta_{a'} + \s_{a'}\} - (\theta_a + \s_a)$ and $\delap_a \coloneqq \max_{a' \in [K]}\{\theta_{a'}-\hicv_{a'} \} - (\theta_a - \hicv_a) = \max_{a' \in [K]}\{\theta_{a'}-\uicv_{a'} \} - (\theta_a - \uicv_a)$ using $\hicv_a = \uicv_a + 1/T$. Since $\icvstar_a \leq \uicv_a \leq \icvstar_a + 1/T$, we have $|\dels_a - \delap_a| \leq 2/T$.

Using the results about $\texttt{UCB}$ algorithm that can be found in \citep[][Theorems~7.1 and~7.2]{lattimore2020bandit}, since $\Alg$ is run in a black-box manner on a shifted bandit instance $\shiftr$ for $T- K\lceil \log_2 T \rceil$ rounds with reward gaps $\delap_a$, we have
\begin{align}
    \nonumber
    R_{\Alg}(T-K\lceil \log_2 T \rceil, \shiftr) &\leq 3 \sum_{\delap_a>0} \delap_a \\
    \nonumber
    &+ 8 \min\left\{\sqrt{(T-K \lceil \log_2 T \rceil) K \log (T- K \lceil\log_2 T \rceil)}\; ; \; \sum_{\delap_a>0} \frac{2 \log (T - K \lceil \log_2 T \rceil)}{\delap_a} \right\} \\
    \nonumber
    &\leq 3 \sum_{a \in [K], \Delta^\star_a>0}\left(\dels_a + \frac{2}{T}\right) + 8 \min\left\{\sqrt{T K \log T}\; ; \; \sum_{a \in [K], \delap>0} \frac{2 \log T}{\delap_a} \right\} \\
    \label{equation:corollary_line_1}
    & \leq 1 + 3 \sum_{a \in [K], \Delta^\star_a>0}\dels_a + 8 \min\left\{\sqrt{T K \log T}\; ; \; \sum_{a \in [K], \delap_a>0} \frac{2 \log T}{\delap_a} \right\} \eqsp,
\end{align}
where the last line holds because of $T \geq 9K > 6 K$.

We now analyse the sum $\sum_{a \in [K], \delap>0} 2 \log T / \delap_a$ and consider two cases: either there exists $\tilde{a} \in [K] \text{ such that } \dels_{\tilde{a}} \leq 4/T$  or not.

\textit{First case:} if there exists $\tilde{a} \in [K]$ such that $\dels_{\tilde{a}} \leq 4/T$, since $T > 9K$, we have for such an action $\tilde{a}$: $2 \log T / \Delta_{\tilde{a}}^\star \geq T \log T / 2> \sqrt{T K \log T}$ as well as $\delap_{\tilde{a}} \leq \dels_{\tilde{a}} + 2 / T \leq 6/T \text{  which is equivalent to  } 2 \log T / \delap_{\tilde{a}} \geq T \log T / 3 \geq \sqrt{T K \log T}$. Consequently,
\begin{equation}
    \label{eq:eq_final_ucb_1}
    \min\left\{\sqrt{T K \log T}\; ; \; \sum_{a \in [K], \delap_a>0} \frac{2 \log T}{\delap_a} \right\} = \sqrt{T K \log T} = \min\left\{\sqrt{TK\log T};  \sum_{a \in [K], \dels_a>0} \frac{2 \log T}{\dels_a}\right\} \eqsp.
\end{equation}

\textit{Second case:} for any $a \in [K], \dels_a > 4/T$. Therefore $\delap_a\geq \dels_a - 2/T > \Delta_a^\star - \Delta_a^\star/2 = \Delta_a^\star/2$. Consequently
\begin{align}
\nonumber
    R_{\Alg}(T-K\lceil \log_2 T \rceil, \shiftr) &\leq 1 + 3 \sum_{a \in [K], \Delta^\star_a>0}\Delta^\star_a + 8 \min\left\{\sqrt{T K \log T}\; ; \; \sum_{a \in [K], \Delta_a^\star>0} \frac{2 \log T}{\delap_a} \right\} \\
    \label{eq:eq_final_ucb_2}
    & \leq 1 + 3 \sum_{a \in [K], \Delta^\star_a>0}\Delta^\star_a + 8\min\left\{\sqrt{T K \log T}\; ; \; \sum_{a \in [K], \dels_a>0} \frac{4 \log T}{\dels_a} \right\} \eqsp.
\end{align}
Finally, combining \eqref{eq:eq_final_ucb_1} and \eqref{eq:eq_final_ucb_2} in \eqref{equation:corollary_line_1} completes the proof.
\end{proof}
\section{Regret Bound for the Contextual Setting}
\label{app:contextual_setting}

For the whole section, we define $\hat{a}_t^\agent \coloneqq \argmax_{a \in \cA_t} \langle \hat{\s}_t, a\rangle$ and recall that $a_t^\agent = \argmax_{a\in \cA_t} \langle \s^\star, a\rangle$. 
 
\subsection{Technical lemmas}

\regretdefinitioncontextual*

\begin{proof}[Proof of Lemma \ref{lemma:regretdefinitioncontextual}] Recall that the regret is defined in \eqref{contextual_regret_def} as $
\regret(T) = \sum_{t=1}^T \mu_t^\star - \E\parentheseDeux{\sum_{t=1}^T (X_{A_t}(t)-\tsfr(t,A_t))}$, where $\mu_t^\star \coloneqq \sup_{\tsfr(t, \cdot) \colon \R^d \to \R_+}\{\langle \theta^\star, a \rangle - \tsfr(t,a)\} \text{ such that } \, a \in \argmax_{a' \in \cA_t}\{ \langle \s^\star, a' \rangle + \tsfr(t,a') \}$. Note that we can write $\mu_t^\star = \sup_{a \in \cA_t, \icv \in \R_+} \{\langle \theta^\star, a \rangle - \1_{\tilde{A}_t}(a, \icv) \icv\}$ where $\tilde{A}_t \coloneqq \{(a, \icv) \colon \langle \s^\star, a \rangle + \icv \geq \max_{a' \in \cA_t} \langle \s^\star, a' \rangle + \1_{a}(a')\icv\}$. First note that if $(a, \icv) \in \tilde{A}_t$, then $\langle \s^\star, a-a' \rangle \geq - \icv$ for any $a' \in \cA_t$, which implies by definition of the optimal incentives \eqref{equation:optimal_incentives_contextual} that $\icvstar(t,a) \leq \icv$. Consequently
\begin{equation*}
    \mu^\star_t = \max_{a \in \cA_t} \{ \langle \theta^\star, a\rangle - \icvstar(t,a)\} 
    = \max_{a \in \cA_t} \{ \langle \theta^\star, a\rangle - \max_{a' \in \cA_t} \{\langle \s^\star,a'\rangle\}+ \langle \s^\star,a\rangle \} \eqsp,
\end{equation*}
hence our result about the regret.
\end{proof}

\controlvolume*

\begin{proof}[Proof of \Cref{lemma:controlvolume}]
    For any $t \in [T], \cS \in \oball(0,1)$ with $s^\star \in \cS$, $\cA_t \in \oball(0,1)$ and $a \in \cA_t$, recall that we defined $a_t^\agent = \argmax_{a' \in \cA_t} \langle \s^\star, a'\rangle$ and $\icvstar(t,a) = \langle \s^\star, a_t^\agent\rangle - \langle \s^\star, a \rangle$. Consequently, defining for any $\s \in \cS$, $a_t^{\s} \coloneqq \argmax_{a' \in \cA_t} \langle \s, a' \rangle$ (the compactness of both $\cS$ and $\cA_t$ as well as the continuity of the applications that we consider guarantee the existence of such an argmax), we have, since $\langle \s^\star, a_t^\agent \rangle \geq \langle \s^\star, a_t^{\s} \rangle$ for any $\s \in \cS$ and associated $a_t^{\s}$,
    \begin{align*}
        \max_{\s \in \cS, a' \in \cA_t} \langle \s, a' - a \rangle - \icvstar(t,a) & = \max_{\s \in \cS} \max_{a' \in \cA_t} \{\langle \s, a' - a \rangle - \langle \s^\star, a_t^\agent - a\rangle \} \\
        & = \max_{\s \in \cS} \{\langle \s, a_t^{\s} - a \rangle - \langle \s^\star, a_t^\agent - a\rangle \} \\
        & \leq \max_{s \in \cS} \{ \langle \s, a_t^{\s} - a\rangle - \langle \s^\star, a_t^{\s} - a\rangle \} \\
        & \leq \max_{s \in \cS} |\langle \s - \s^\star, a_t^{\s} \rangle | + \max_{s \in \cS} |\langle \s - \s^\star, a \rangle | \\
        & \leq 2 \, \diam(\cS, \cA_t) \eqsp.
    \end{align*}
Similarly, we have
\begin{align*}
    \icvstar(t,a) - \max_{\s \in \cS, a' \in \cA_t} \langle \s, a' - a \rangle 
    & \leq \langle \s^\star, a_t^\agent - a\rangle - \max_{\s \in \cS} \langle \s, a_t^\agent - a \rangle \\
    & \leq \max_{s \in \cS} |\langle \s^\star - \s, a_t^\agent \rangle| + \max_{s \in \cS}|\langle \s^\star - \s, a \rangle| \\
        & \leq 2 \, \diam(\cS, \cA_t) \eqsp,
\end{align*}
and the proof follows.
\end{proof}

\enoughincentivescontextual*

\begin{proof}
The proof is similar to the proof of \Cref{lemma:controlvolume}. Consider $t \geq 1$, $\cA_t \subseteq \oball(0,1)$, $\cS_t \subseteq \oball(0,1)$ such that $\cE_t$ holds, $a_t \in \cA_t$. Then $1/T > \max_{a_t^1 \ne a_t^2 \in \cA_t} \diam\left(\cS_t, (a_t^1-a_t^2)/\|a_t^1-a_t^2\|\right)$. Recall that we defined $a_t^\agent = \argmax_{a\in \cA_t} \langle \s^\star, a\rangle$ and $\hat{a}_t^\agent = \argmax_{a \in \cA_t} \langle \hat{\s}_t, a\rangle$, $\icvstar(t,a_t) = \langle \s^\star, a_t^\agent \rangle - \langle \s^\star, a_t \rangle $ and $\hicv(t,a_t) = \langle \hat{\s}_t, \hat{a}_t^\agent \rangle - \langle \hat{\s}_t, a_t \rangle + 2/T$, giving
    \begin{equation*}
        \hicv(t,a_t) \geq \langle \hat{\s}_t, a_t^\agent \rangle - \langle \hat{\s}_t, a_t \rangle + 2/T \eqsp.
    \end{equation*}
Therefore, under $\cE_t$, it holds
    \begin{align}
    \nonumber
    \hicv(t,a_t) - \icvstar(t, a_t) &\geq \langle \hat{\s}_t, a_t^\agent \rangle - \langle \hat{\s}_t, a_t \rangle + \frac{2}{T} - \langle \s^\star, a_t^\agent  \rangle + \langle \s^\star, a_t \rangle \\
    \nonumber
    & = \langle \hat{\s}_t - \s^\star,a_t^\agent  - a_t \rangle + \frac{2}{T} \\
    \label{equation:lemma_enough_incentives_ineq_1}
    & > - \diam\left(\cS_t, \frac{a_t^\agent -a_t}{\|a_t^\agent -a_t\|}\right) \underbrace{\|a_t^\agent -a_t\|}_{\leq 2} + 2 \max_{a_t^1 \ne a_t^2 \in \cA_t} \diam\left(\cS_t, \frac{a_t^1-a_t^2}{\|a_t^1-a_t^2\|}\right) \geq 0 \eqsp.
    \end{align}
Similarly, since $\langle \s^\star, a_t^\agent \rangle \geq \langle \s^\star, \hat{a}_t^\agent \rangle$, under $\cE_t$, we have
\begin{align}
\nonumber
\hicv(t,a_t)- \icvstar(t,a_t)
&= \langle\hat{\s}_t, \hat{a}_t^\agent - a_t\rangle + \frac{2}{T} - \langle \s^\star, a_t^\agent - a_t\rangle  \\
\nonumber
&\leq \langle\hat{\s}_t, \hat{a}_t^\agent - a_t\rangle - \langle \s^\star, \hat{a}_t^\agent - a_t \rangle  + \frac{2}{T}  = \langle \hat{\s}_t - \s^\star, \hat{a}_t^\agent - a_t\rangle + \frac{2}{T} \\
\label{equation:lemma_enough_incentives_ineq_2}
& \leq \| \hat{a}_t^\agent - a\| \cdot  \diam\left(\cS_t,\frac{ \hat{a}_t^\agent - a_t}{\| \hat{a}_t^\agent - a_t\|}\right) + \frac{2}{T} \leq \frac{4}{T} \eqsp.
\end{align}
Combining \eqref{equation:lemma_enough_incentives_ineq_1} and \eqref{equation:lemma_enough_incentives_ineq_2} with the definition of $\estsfr$, we obtain the result.
\end{proof}

%We restate here the lemma of \citet[][Lemma 6.3]{lobel2018multidimensional}.
%\begin{lemma}\label{lemma:contains_ball} If $\cS \subseteq \oball(0,1)$ is a convex body such that $\diam(\cS, u) \geq \pvacst$ for any unit vector $u$, then $\cS$ contains a ball of diameter $\pvacst/d$. \end{lemma}

\begin{lemma}
\label{lemma:lower_bound_volume}
Let $t \in [T]$ such that $\cE_t$ does not hold. Then $\Vol(\Pi_{\VV_t^\perp}\cS_t) \geq \pvacst^{2d} / d^{2d}$, where $\pvacst = 1/16T^2 d(d+1)^2$ and $\VV_t$ is defined in \eqref{equation:definition_V_t}.
\end{lemma}

\begin{proof}
Since $\cE_t$ does not hold, there exists a direction $u \in \cA_t \subseteq \oball(0,1)$ such that $\diam(\cS_t, u) \geq 1/T \geq \pvacst$. For any $u \in \VV_t^\perp, \diam(\cS_t, u) \geq \pvacst$. Lemma 6.3 of \citet[][Section 6]{lobel2018multidimensional} guarantees that $\Pi_{\VV_t^\perp}(\cS_t)$ contains a $k$-dimensional ball of radius $\pvacst/k$, where $k\coloneqq \dim(\VV_t^\perp)$. Therefore, $\Vol(\Pi_{\VV_t^\perp}(\cS_t)) \geq \pvacst^k \pi^{k/2}/k^k \Gamma(k/2 +1)$, where $\Gamma$ stands for Euler's Gamma function \citep{smith1989small}. Therefore, we have by definition of Euler's Gamma function: $\Vol(\Pi_{\VV_t^\perp}(\cS_t)) \geq \pvacst^k \pi^{k/2} / k^k k^k \geq \pvacst^{2d} / d^{2d}$.
\end{proof}

\llboundedE*

\begin{proof}[Proof of Lemma \ref{lemma:boundedE}]
This proof follows the same line as the proof of the main theorem of \citet[][Section 7]{lobel2018multidimensional}.
Let $t\in [T]$ such that $\cE_t$ does not hold.
We define $w_t$ as
\begin{equation}
\label{equation:definition_w_t}
w_t \coloneqq \argmax\defEns{\diam \left(\cS_t, \frac{a_t^1-a_t^2}{\|a_t^1-a_t^2\|}\right) \,: \, \frac{a_t^1-a_t^2}{\|a_t^1-a_t^2\|} \text{ such that }a_t^1\ne a_t^2 \in \cA_t}\eqsp,
\end{equation}
where $\cS_t$ is defined in \eqref{equation:definition_S_t}. Our goal is to bound the number of steps for which the diameter of $\cS_t$ in the direction $w_t$ is strictly superior to $1/T$. Define
\begin{equation*}
    \counE_t \coloneqq \1 \{ \diam(\Cyl(\cS_t, \VV_t^\perp), w_t) \geq 1/T \} \eqsp,
\end{equation*}
where $\VV_t$ is defined in \eqref{equation:definition_V_t}. Let $\Pi_{\mathsf{E}}$ denote the orthogonal projection onto the subspace $\mathsf{E} \subset \rset^d$ and $\Vol(\Pi_{\mathsf{E}} \cS) = \mu_{\mathsf{E}}(\Pi_E\cS)$, where $\mu_{\mathsf{E}}$ is the Lebesgue measure of $\Pi_E\cS$, well-defined for $\Pi_E\cS$ being a convex body.
%Haar measure of $\mathsf{E}$; see \citet[Chapter XI]{halmos2013measure}.
Setting $\pvacst = T^{-2}/16 d(d+1)^2$, we can apply the projected Grünbaum lemma of \citet[][Lemma 7.1]{lobel2018multidimensional} to obtain that
\begin{equation*}
    \Vol(\Pi_{\VV_t^\perp} \cS_{t+1}) \leq \left( 1 -\rme^{-2} \right)^{\counE_t} \Vol(\Pi_{\VV_t^\perp} \cS_t) \eqsp.
\end{equation*}

By definition of $\VV_t$ in \eqref{equation:definition_V_t}, we have for any $u \in \VV_t$, $\diam(\cS_t, u) \geq \pvacst$. Therefore, by definition of $\cS_{t+1}$ in \eqref{equation:definition_S_t}, the directional Grünbaum Theorem \citep[][Theorem 5.3]{lobel2018multidimensional} guarantees that we have: $\diam(\cS_{t+1},u) \geq \pvacst/(d+1)$. Note that for $\cS_t$ being a convex body, \Cref{lemma:convex_body_non_empty} ensures that $\cS_{t+1}$ remains a convex body.%, as well as $\Pi_{\VV_t^{(i+1), \perp}} \cS_{t+1}$.

If $J_{t+1}=0$, where $J_t$ is defined in \eqref{equation:definition_J_t}, then $\VV_{t+1}^\perp = \VV_t^\perp$, and we have $\Vol(\Pi_{\VV_{t+1}^\perp} \cS_{t+1}) = \Vol(\Pi_{\VV_t^\perp} \cS_{t+1})$.

Otherwise, let $i \in [J_{t+1}-1]$ and $v \in \VV_t^{(i), \perp}$ such that $\VV_t^{(i+1)} = \VV_t^{(i)} \cup \{v\}$. We have $\VV_t^{(i),\perp}\cap \sppp(v)^\perp = \{x \in \VV_t^{(i),\perp} :~ v^\top x = 0 \} \subseteq \VV_t^{(i+1),\perp} \subseteq \VV_t^{(i),\perp}$ and $\dim(\VV_t^{(i+1),\perp}) = \dim(\VV_t^{(i),\perp}) -1$. %If we add a new direction $v \in \VV_t^{(i), \perp}$ to $\VV_t^{(i)}$, then we replace $\Pi_{\VV_t^{i,\perp}} \cS_t$ by its projection onto the subspace $\VV_t^{(i),\perp}\cap \sppp(v)^\perp = \{x \in \VV_t^{(i),\perp} :~ v^\top x = 0 \} \subseteq \VV_t^{(i+1),\perp}$.
Then, applying the Cylindrification Lemma from \citet[][Lemma 6.1]{lobel2018multidimensional} we obtain
\begin{equation*}
    \Vol(\Pi_{\VV_t^{(i+1), \perp}} \cS_{t+1}) \leq \frac{d(d+1)^2}{\pvacst} \Vol(\Pi_{\VV_t^{(i), \perp}} \cS_{t+1}) \eqsp.
\end{equation*}
If $J_{t+1}=r$, %we need to add $r$ new directions to $\VV_t$,
the volume can blow up by at most
$\left( d(d+1)^2 / \pvacst \right)^r$. In particular, since the initial volume is bounded by $\Vol(\oball(0,1)) \leq 8\pi^2 /15 \leq 6$ \citep{smith1989small}, then by \Cref{lemma:lower_bound_volume}, we obtain: $\pvacst^{2d} / d^{2d} \leq \Vol(\Pi_{L_t}\cS_t) \leq 6 \cdot \left( d(d+1)^2/\pvacst \right)^d \cdot \left( 1 - 1/\rme^2 \right)^{\sum_{t=1}^T \counE_t} \,$. Therefore, applying the logarithm function, we obtain
\begin{align*}
&\sum_{t=1}^T \counE_t \leq -1/\log(1-\rme^{-2}) \left(\log 6 + 2d \log (16) + 5d\log(d) + 6d\log(d+1) + 4d \log (T) \right)\eqsp,
\end{align*}
giving, since $-1/\log(1-\rme^{-2}) < 16$:
$\sum_{t=1}^T \1\left\{\diam(\Cyl(\cS_t, L_t), w_t) \geq 1/T\right\} \leq 192 d\log \left( dT\right)$. Therefore, \cref{lem:cyl_properties} ensures that $\diam(\cS_t, w_t) \leq \diam(\Cyl(\cS_t, \VV_t^\perp), w_t)$ and we get
    \begin{align*}
    & \sum_{t=1}^T \1\left\{\max_{s \in \cS_t}|\langle s, w_t\rangle| \geq 1/T\right\} \leq 192 \, d\log \left(dT\right) \eqsp,
    \end{align*}
which concludes the proof by definition of $w_t$ in \eqref{equation:definition_w_t} and $\cE_t$ in \eqref{equation:definition_cE_t}.
\end{proof}

\shiftedrewardbound*

\begin{proof}[Proof of \Cref{lemma:shifted_reward_bound}]

Consider $t \in I_T$, $a \in \cA_t$: by \eqref{equation:definition_I_t}, $\cE_t$ is true. Consider $\epscorruption_t$ as defined in \eqref{equation:definition_espcorruption}

Using \Cref{lemma:enough_incentives_contextual} gives $
    |\epscorruption_t| \leq 4/T \eqsp,$
and summing over all the iterations $t \in I_T$, since $|I_T| \leq T$, we have $\sum_{t \in I_t} |\epscorruption_t| \leq 4 \cdot 1/T \cdot |I_T| \leq 4$.
\end{proof}

\begin{lemma}
\label{lemma:good_action_chosen_contextual}
For any $t \in [T]$, $\epsilon>0$ and action $a \in \cA_t$, set $\icvstare(t,a) \coloneqq \max_{a_t^\agent \in \cA_t} \langle \s^\star, a_t^\agent \rangle - \langle \s^\star, a\rangle +\epsilon$, and define the incentive function $\tsfrstare_a(t,a')=\indi{a}(a')\icvstare(t,a)$ for any $a' \in \cA_t$. Then $A_t=a$, where $A_t$ is defined by \eqref{eq:contextual_action}.
\end{lemma}
\begin{proof}[Proof of \Cref{lemma:good_action_chosen_contextual}.]
Note that for any $a' \in \cA_t, a' \ne a$, $\langle \s^\star, a'\rangle + \tsfrstare_a(t,a') < \langle \s^\star, a\rangle + \tsfrstare_a(t,a)$ and, as a result, $A_t = a$.
\end{proof}

\subsection{Proof of \Cref{theorem:regret_bound_contextual}}

\begin{proof}[Proof of \Cref{theorem:regret_bound_contextual}.]
Denote for any $t \in [T]$, $a_t^\principal \coloneqq \argmax_{a \in \cA_t} \{ \langle \theta^\star, a \rangle - \icvstar(t,a) \} = \argmax_{a\in \cA_t} \langle \theta^\star + \s^\star, a\rangle, a_t^\agent \coloneqq \argmax_{a\in \cA_t} \langle \s^\star, a\rangle$ and $\hat{a}_t^\agent = \argmax_{a\in \cA_t} \langle \hat{\s}_t, a\rangle$.
%Remember that $\cE_t$ was defined as follows \begin{equation*} \cE_t = \left\{ \max_{a_t^1\ne a_t^2 \in \cA_t} \diam \left(\cS_t, \frac{a_t^1-a_t^2}{\|a_t^1-a_t^2\|}\right) < \precision \right\} \end{equation*}
Note that if $\cE_t$ holds, for any $a_t^1, a_t^2 \in \cA_t, a_t^1 \ne a_t^2$, \Cref{assumption:bounded_in_unit_ball} gives that if
\begin{align*}
\diam\left(\cS_t, \frac{a_t^1-a_t^2}{\|a_t^1-a_t^2\|}\right) < \precision
\; , \text{ then } \; & \max_{s\in \cS_t} \langle s, a_t^1 - a_t^2 \rangle = \max_{s\in \cS_t} \ps{s}{\frac{a_t^1 - a_t^2}{\|a_t^1 - a_t^2\|}} \underbrace{\|a_t^1 - a_t^2\|}_{\leq 2} < 2 \cdot \precision \eqsp,
\end{align*}
and therefore $\diam\left(\cS_t, a_t^1-a_t^2\right) < 2/T$.

If $\cE_t$ does not hold, the incentive function proposed in \cref{algorithm:projected_volume} is given by $\tsfr(t, a)= 3\cdot \indi{a_t^1}(a) +(3+\langle \hat{\s}_t,a_t^1-a_t^2\rangle) \cdot  \indi{a_t^2}(a)$. Therefore: $\tsfr(t, a_t^1) = 3, \tsfr(t, a_t^2) = 3 + \langle \hat{\s}_t,a_t^1-a_t^2\rangle \leq 5$. When $\cE_t$ holds, the  incentive function proposed in $\calgU$ is defined as $\tsfr(t,a) = \1_{\Arec_t}(a)\hicv(t,\Arec_t) \leq 2$.

For any $t \in [T]$, we define the instantaneous regret reg$_t$ at $t$ as
\begin{equation*}
    \text{reg}_t \coloneqq \mu_t^\star - (X_{A_t}(t) - \tsfr(t,A_t)) \eqsp,
\end{equation*}
where $\mu_t^\star$ is defined in \eqref{equation:definition_mu_t} and we decompose the regret into two terms, making use of the Cauchy-Schwarz inequality as well as \Cref{assumption:bounded_in_unit_ball} to obtain
\begin{align*}
\regret(T) &= \E \parentheseDeux{\sum_{t=1}^T \1_{\{\cE_t\}}\reg_t + \sum_{t=1}^T \1_{\{\cE^\complementary_t\}}\reg_t} \\
& \leq \E \parentheseDeux{\sum_{t=1}^T \1_{\{\cE_t\}}\reg_t} + \E \parentheseDeux{ \sum_{t=1}^T \1_{\{\cE^\complementary_t\}}(\max_{a \in \cA_t} \langle \theta^\star + \s^\star, a\rangle - \max_{a' \in \cA_t} \langle \s^\star, a'\rangle - \langle \theta^\star, A_t\rangle + \tsfr(t,A_t))} \\
& \leq \E \parentheseDeux{\sum_{t=1}^T \1_{\{\cE_t\}}\reg_t} + \E \parentheseDeux{ \sum_{t=1}^T \1_{\{\cE_t^\complementary\}}(\max_{a \in \cA_t} \{\underbrace{\langle \theta^\star,a\rangle}_{\leq 1}+ \underbrace{\langle \s^\star, a-a_t^\agent \rangle\}}_{\leq 0} \underbrace{- \langle \theta^\star, A_t\rangle }_{\leq 1}+ \underbrace{\tsfr(t,A_t))}_{\leq 5}} \\
&\leq \E \parentheseDeux{\sum_{t=1}^T \1_{\{\cE_t\}}\reg_t} + 7 \, \E \parentheseDeux{\sum_{t=1}^T \1_{\{\cE^\complementary_t\}}} \eqsp.
\end{align*}
Using \Cref{lemma:boundedE}, we can bound the second term
\begin{equation*}
\E \parentheseDeux{\sum_{t=1}^T \1_{\{\cE_t^\complementary\}}} \leq 192 d\log (d T) \eqsp.
\end{equation*}
Now we bound the first term. Working on steps $t$ such that $\cE_t$ is true with incentive function $\estsfr_{\Arec_t}(t,\cdot) = \1_{\Arec_t}(\cdot)\hicv(t, \Arec_t)$, \Cref{lemma:enough_incentives_contextual} guarantees that $A_t = \Arec_t$, and we have
\begin{align*}
&\E \parentheseDeux{\sum_{t=1}^T \1_{\{\cE_t\}}\reg_t}
\\
&= \E\parentheseDeux{\sum_{t=1}^T\1_{\{\cE_t\}}\left(\max_{a \in \cA_t} \{\langle \theta^\star + \s^\star, a\rangle - \max_{a' \in \cA_t} \langle \s^\star, a'\rangle\} - \left\{\langle \theta^\star, \Arec_t \rangle -\left( \langle \hat{\s}_t, \hat{a}_t^{\agent}\rangle - \langle \hat{\s}_t, \Arec_t \rangle + \frac{2}{T}\right)\right\}\right)} \\
& = \E\parentheseDeux{\sum_{t=1}^T \1_{\{\cE_t\}} \left(\langle \theta^\star , a_t^\principal \rangle + \langle \s^\star , a_t^\principal\rangle - \langle \s^\star, a_t^\agent \rangle - \langle \theta^\star , \Arec_t \rangle + \langle \hat{\s}_t , \hat{a}_t^\agent \rangle - \langle \hat{\s}_t , \Arec_t \rangle \right) + \sum_{t=1}^T \1_{\{\cE_t\}} \frac{2}{T} }  \\
&\leq \E\parentheseDeux{\sum_{t=1}^T \1_{\{\cE_t\}}\left(\langle \theta^\star , a_t^\principal \rangle + \langle \s^\star , a_t^\principal\rangle - \langle \s^\star, a_t^\agent \rangle - \langle \theta^\star , \Arec_t \rangle \right)} +2T \precision \\
& \quad + \E\parentheseDeux{\sum_{t=1}^T \1_{\{\cE_t\}} \left( - \langle \s^\star, \Arec_t\rangle + \langle \s^\star, a_t^\agent \rangle  + \langle \hat{\s}_t , \hat{a}_t^\agent \rangle - \langle \hat{\s}_t , \Arec_t \rangle + \langle \s^\star, \Arec_t\rangle - \langle \s^\star, a_t^\agent \rangle \right)} \\
& = \E\parentheseDeux{\sum_{t=1}^T \1_{\{\cE_t\}} \left(\langle \theta^\star , a_t^\principal \rangle + \langle \s^\star , a_t^\principal\rangle - \langle \s^\star, a_t^\agent \rangle - \langle \theta^\star , \Arec_t \rangle - \langle \s^\star, \Arec_t\rangle + \langle \s^\star, a_t^\agent \rangle \right)} \\
& \quad + \E\parentheseDeux{\sum_{t=1}^T \1_{\{\cE_t\}} \left(\langle \hat{\s}_t , \hat{a}_t^\agent \rangle + \langle \s^\star - \hat{\s}_t, \Arec_t\rangle - \langle \s^\star, a_t^\agent \rangle \right)} + 2 \eqsp.
\end{align*}
Since $\langle \s^\star, a_t^\agent \rangle \geq \langle \s^\star, \hat{a}_t^\agent \rangle$, we have $- \langle \s^\star, a_t^\agent \rangle \leq - \langle \s^\star, \hat{a}_t^\agent \rangle$. Therefore
\begin{align*}
&\E\parentheseDeux{\sum_{t=1}^T \1_{\{\cE_t\}} \left(\langle \hat{\s}_t , \hat{a}_t^\agent \rangle + \langle \s^\star - \hat{\s}_t, \Arec_t\rangle - \langle \s^\star, a_t^\agent \rangle \right)}
\\
&\leq \E\parentheseDeux{\sum_{t=1}^T \1_{\{\cE_t\}} \left(\langle \hat{\s}_t , \hat{a}_t^\agent \rangle + \langle \s^\star - \hat{\s}_t, \Arec_t\rangle - \langle \s^\star, \hat{a}_t^\agent \rangle \right)} \\
& = \E\parentheseDeux{\sum_{t=1}^T \1_{\{\cE_t\}} \left( \langle \s^\star - \hat{\s}_t, \Arec_t\rangle - \langle \s^\star - \hat{\s}_t, \hat{a}_t^\agent \rangle \right)} \\
& = \E\parentheseDeux{\sum_{t=1}^T \1_{\{\cE_t\}} \left( \langle \s^\star - \hat{\s}_t, \Arec_t - \hat{a}_t^\agent \rangle \right)} \\
& \leq \sum_{t=1}^T \1_{\{\cE_t\}} \1\{\Arec_t \ne \hat{a}_t^\agent\} \underbrace{\|\Arec_t - \hat{a}_t^\agent \|}_{\leq 2} \diam\left(\cS_t, \frac{\Arec_t - \hat{a}_t^\agent}{\|\Arec_t - \hat{a}_t^\agent\|}\right)  < 2 T \cdot \precision  = 2 \eqsp,
\end{align*}
and plugging this inequality gives
\begin{align*}
\E \parentheseDeux{\sum_{t=1}^T \1_{\{\cE_t\}}\reg_t}
& \leq \E\parentheseDeux{\sum_{t=1}^T \1_{\{\cE_t\}} \left(\langle \theta^\star + \s^\star , a_t^\principal\rangle - \langle \s^\star, a_t^\agent \rangle - \langle \theta^\star + \s^\star, \Arec_t\rangle + \langle \s^\star, a_t^\agent \rangle \right)} +4 \\
& = \E\parentheseDeux{\sum_{t \in I_T} \max_{a\in \cA_t}\{ \langle \theta^\star + \s^\star, a\rangle - \max_{a' \in \cA_t} \langle \s^\star, a'\rangle \} } - \E\parentheseDeux{\sum_{t \in I_T}\langle \theta^\star+ \s^\star, \Arec_t\rangle - \max_{a\in \cA_t} \langle \s^\star, a\rangle} +4 \\
& = R^{\text{corrupt}}_{\CAlg}(I_T, \epscorruption_{I_T}) + 4 \eqsp,
\end{align*}
where $R^{\text{corrupt}}_{\calgU}$ is defined in \eqref{equation:Rcorrupt_definition} with $\icvstar(t,a) = \max_{a' \in \cA_t} \ps{\s^\star}{a'} - \ps{\s}{a}$. Plugging all the terms together gives the following upper-bound for the regret:
\begin{equation*}
\regret(T) \leq 1344 \, d\log(dT) + 4 + R^{\text{corrupt}}_{\CAlg}(I_T, \epscorruption_{I_T})  \eqsp.
\end{equation*}
\end{proof}

\begin{proof}[Proof of \Cref{corollary:regret_bound_contextual}]
    Here, \Cref{lemma:shifted_reward_bound} guarantees that $\corruption = 4$. With the setup of \citet{he2022nearly}, we take $L=1, S=1, R=1$, and $\alpha = \sqrt{d}/4$.
    
    Choose $\lambda = 1$ as a regularization parameter and $\probreg = 1/T$ as a confidence level. Using $\Corralg$ proposed in \citet{he2022nearly} as a subroutine robust to corruption in the stochastic linear case, since our subroutine is fed with the same reward $r'$ as in their model while $R^{\text{corrupt}}_{\CAlg}(I_T, \epscorruption_{I_T})$ is defined compared to the true reward $r$, we can use the result provided in \citet[][Theorem 4.2]{he2022nearly} to bound $R^{\text{corrupt}}_{\CAlg}(I_T, \epscorruption_{I_T})$ with probability $1 - 1/T$ and use the fact that our instantaneous regret is always bounded by $7$ as it is shown in the proof of \Cref{theorem:regret_bound_contextual} to get in expectation for some universal constant $B>0$ 
    \begin{align*}
        R^{\text{corrupt}}_{\CAlg}(I_T, \epscorruption_{I_T})  &\leq (1- \probreg)B \left(2d\sqrt{T}\log\left(\frac{1+T}{\probreg}\right)+ \sqrt{d\lambda T} \sqrt{\log \left(1+ T \right)} + 4 d \sqrt{\log\left(\frac{1+T}{\probreg}\right)^3}\right) + 7T \probreg \\
        & \leq \left(1-\frac{1}{T}\right)B \left(2d\sqrt{T}\log\left(\frac{1+T}{\frac{1}{T}}\right)+ \sqrt{d T} \sqrt{\log \left(\frac{1+ T}{\frac{1}{T}} \right)} +4d \sqrt{\log\left(\frac{1+T}{\frac{1}{T}}\right)^3}\right) +7 T \frac{1}{T}\\
        & \leq \left(1-\frac{1}{T}\right) B \left(2 d\sqrt{T}\log \left(T+T^2 \right) + \sqrt{dT}\sqrt{\log \left(T+T^2 \right)} + 4 d \left(\log \left(T+T^2 \right)\right)^{\frac{3}{2}} \right) +7 \eqsp,
    \end{align*}
    therefore, since $\corruption \leq 4$, there exists a constant $C_{\CAlg}$ such that
    \begin{equation*}
    R^{\text{corrupt}}_{\CAlg}(I_T, \epscorruption_{I_T})  \leq 7 + C_{\CAlg} d \sqrt{T} \log T \eqsp.
    \end{equation*}
    Finally plugging this term in the bound from \Cref{theorem:regret_bound_contextual} and integrating the $3$ factor in constant $B$ gives the result
    \begin{equation*}
    \regret(T) \leq 11 + 1344 \, d \log (dT) + C_{\CAlg} d \sqrt{T} \log T \eqsp.
    \end{equation*}
\end{proof}

\section{Algorithms}
\label{appendix:algorithms}
\subsection{\texttt{UCB} subroutine}

We present the \texttt{Binary search} subroutine and the $\texttt{UCB}$ algorithm that we use as a subroutine in $\algU$ as formulated in \citet[][Algorithm~3]{lattimore2020bandit}.

\begin{algorithm}[!ht]
\caption{\texttt{Binary Search Subroutine}}\label{algorithm:binary_search}
\begin{algorithmic}[1]
    \STATE {\bfseries Input:} action $a, \bsrounds$ 
    \STATE {\bfseries Initialize:} $\licv_a(0), \uicv_a(0)=0,1$
    \FOR{$t = 1, \ldots,\bsrounds$}
        \STATE $\icvmid_a(t-1) = \frac{\uicv_a(t-1) + \licv_a(t-1)}{2}$
        \STATE Propose incentive $\icvmid_a(t-1)$ on arm $a$
        \STATE \textbf{If} $A_{t-1}=a$ \textbf{then} $\uicv_a(t) = \icvmid_a(t-1)$ and $\licv_a(t) = \licv_a(t-1)$ \textbf{else} $\: \licv_a(t) = \icvmid_a(t)$ and $\uicv_a(t) = \uicv_a(t-1)$
    \ENDFOR
    \STATE {\bfseries Return} $\licv_a(\bsrounds), \uicv_a(\bsrounds)$
\end{algorithmic}
\end{algorithm}

\begin{algorithm}[!ht]
\caption{\texttt{UCB Subroutine}}\label{algorithm:ucb}
\begin{algorithmic}[1]
    \STATE {\bfseries Input:} Set of arms $K$, horizon $T$
    \STATE {\bfseries Initialize:} For any arm $a \in [K]$, set $\emprew_a \coloneqq 0, \, T_a \coloneqq 0$%, history $\cH \coloneqq \emptyset$.
    \FOR{$1 \leq t \leq K$:}
        \STATE Pull arm $a = t$
        \STATE Update $\emprew_a = X_a(t), \, T_a(t)= 1$
    \ENDFOR
    \FOR{$t \geq K+1$}
        \STATE Pull arm $a_{\text{max}} \in \argmax_{a \in [K]} \left\{\emprew_a(t-1) + 2\sqrt{\frac{\log T}{T_a(t-1)}} \right\}$
        \STATE Update $T_a(t)=T_a(t-1)+1, \, \emprew_a(t) = \frac{1}{T_a(t)}(T_a(t-1)\emprew_a(t-1) + X_a(t))$
    \ENDFOR
\end{algorithmic}
\end{algorithm}

\subsection{\texttt{Projected volume algorithm}}

We present the \texttt{Projected volume algorithm} from \citet{lobel2018multidimensional} that we use as a subroutine in $\calgU$. For any horizon $T$ and $0 < \pvacst <T^{-2}/(16d(d+1)^2)$, this algorithm defines recursively a sequence $(\cS_t, \VV_t)_{t \in [T]}$ such that $(\cS_t)_{t \in [T]}$ is a sequence of decreasing subsets (for the inclusion) of $\oball(0,1)$ including $\s^\star$ and $(\VV_t)_{t \in [T]}$ is a sequence of increasing sets of $\rset^d$ containing orthogonal directions $\{v_i\}_{i \in [n]}$ along which the principal has a good knowledge of $\langle \s^\star, v_i\rangle$.

The main ingredient in \citet{lobel2018multidimensional} allowing low regret is \textit{cylindrification}. Given a compact convex set $\cS \subseteq \R^d$, $\VV = \{v_1, \ldots, v_n\}$, and $\Pi_{\VV^\perp}(\cS)$ being the orthogonal projection of $\cS$ onto $\VV^\perp$, we define the cylindrification of $S$ on $\VV$ as
\begin{align*}
\Cyl(S, \VV) &\coloneqq \Pi_L(\cS) + \Pi_{\text{span}(v_1)}(\cS) + \ldots + \Pi_{\text{span}(v_n)}(\cS) \\
&=\left\{ x + \sum_{i=1}^n y_i v_i \colon x \in \Pi_L(\cS),\min_{s \in \cS}\langle s, v_i \rangle \leq y_i \leq \max_{s \in \cS}\langle s, v_i \rangle \right\} \eqsp.
\end{align*}
At iteration $t$, given $(\cS_t, \VV_t)$, we define the estimate $\hat{\s}_t$ of $\s^\star$ as the centroid of $\Cyl(\cS_t, \VV_t)$:
\begin{equation}
\label{equation:centroid_definition}
    \hat{\s}_t \coloneqq \frac{1}{\Vol(\Cyl(\cS_t, \VV_t))} \int_{\Cyl(\cS_t, \VV_t)} x \rmd x 
    \eqsp,
\end{equation}
Note that the iterative construction of $\cS_t$ described below together with \Cref{lemma:convex_body_non_empty} guarantees that $\cS_t$ is always a convex body, making $\Vol(\cS_t)$ well-defined as the Lebesgue measure of $\cS_t$ in dimension $d$ and $\Vol(\cS_t)>0$. Combined with \Cref{lem:cyl_properties}, it guarantees that $\Vol(\Cyl(\cS_t, \VV_t))$ is well-defined, and $\Vol(\Cyl(\cS_t, \VV_t))>0$. Therefore, \eqref{equation:centroid_definition} is well-defined.

At iteration $t$, given $(\cS_t, \VV_t)$ and two actions $a_t^1, a_t^2 \in \cA_t$, recall that we defined $w_t$ as
\begin{equation*}
w_t \coloneqq \argmax\defEns{\diam \left(\cS_t, \frac{a_t^1-a_t^2}{\|a_t^1-a_t^2\|}\right) \,: \, \frac{a_t^1-a_t^2}{\|a_t^1-a_t^2\|} \text{ such that }a_t^1\ne a_t^2 \in \cA_t}\eqsp.
\end{equation*}
Then, the principal offers the incentive function $\tsfr(t,a)= 5 \cdot \indi{a_t^1}(a) +(5+\langle \hat{\s},a_t^1-a_t^2\rangle) \cdot  \indi{a_t^2}(a)$. Recall that $\tsfr(t,a)$ is always bounded by 5 and that $a_t^1, a_t^2$ are chosen such that $\langle \hat{\s}_t, a_t^1 - a_t^2\rangle \geq 0$, ensuring that we either have $A_t = a_t^1$ or $A_t = a_t^2$.

If $A_t = a_t^1$, it means that $\langle \s^\star, w_t \rangle \geq 0$ and we update $\cS_{t+1} = \cS_t \cap \{\s | \langle \s, w_t \rangle\geq x_t\}$. Otherwise, if $A_t = a_t^2$, it means that $\langle \s^\star, w_t \rangle \leq 0$ and we update $\cS_{t+1} = \cS_t \cap \{\s | \langle \s, w_t \rangle\leq x_t\}$. This defines the subset $\cS_{t+1}$ such that $\s^\star \in \cS_{t+1}$:
\begin{equation}
\label{equation:definition_S_t}
    \cS_{t+1} = \cS_t \cap (\1_{a_t^1}(A_t) \{\s | \langle \s, w_t \rangle\geq x_t\} + \1_{a_t^2}(A_t)\{\s | \langle \s, w_t \rangle\leq x_t\})
\end{equation}

Regarding the subspace $\VV_{t+1}$, we consider $\VV_{t+1} = \VV_{t+1}^{J_{t+1}}$ where $(\VV_{t+1}^i)_{i \in \N}$ is defined as
\begin{align}
    \nonumber
    \VV_{t+1}^{(0}) &\coloneqq \VV_t \\
    \label{equation:definition_V_t}
    V_{t+1}^{(i+1)} &\coloneqq \begin{cases} V_{t+1}^{(i)} \cup \{v\} \text{  if  } \exists \, v \in (\VV_{t+1}^{(i)})^\perp \; \colon \; \diam(\cS_{t+1}, v) \leq \pvacst \\ \VV_{t+1}^{(i)} \text{  otherwise} \end{cases} \\
    \label{equation:definition_J_t}
    J_{t+1} &\coloneqq \min \{ i \colon \text{ it does not exist } v \in (\VV_{t+1}^{(i)})^\perp \text{  such that  } \diam(\cS_{t+1}, v) \leq \pvacst \} \eqsp,
\end{align}
which exists since if it does not exist $i$ such that it does not exist $v \in (\VV_{t+1}^{(i)})^\perp \text{  such that  } \diam(\cS_{t+1}, v) \geq \pvacst$, we would have $\dim(\sppp(\VV_t^{(i+1)})) = \dim (\sppp (\VV_t^{(i)})) +1$ for any $i \in \N$, which would imply a contradiction.

In what follows, we provide technical results ensuring that $\cS_t$ is a convex body for any $t \in [T]$ and $\cS_t \subseteq \Cyl(\cS_t, \VV_t)$, which implies that $\Cyl(\cS_t, \VV_t)$ has non-empty interior.

\begin{lemma}
\label{lemma:convex_body_non_empty}
    Let $\cS$ be a convex body in $\R^d$ and $\s$ be a point in the interior of $\cS$: $\s \in \mathring{\cS}$. Let $\GG$ be the half-space defined by $\GG \coloneqq \{ x \in \R^d \colon \langle h^\star, x\rangle \geq 0 \}$,  for some $h^\star \in \R^d$. Suppose that $\s \in \msg$. Then $\cS \cap \GG$ is a convex body.
\end{lemma}

\begin{proof}
Note that we only need to show that $\cS\cap \msg$ has non-empty interior since the intersection of two compact convex sets is compact and convex.

The only case that we consider is $\s \in\msh$ where $\HH$ is the hyperplane defined by $\HH \coloneqq \{ x \in \R^d \colon \langle h^\star, x\rangle =0 \}$. The other case simply follows from the fact that the intersection of two open sets is also open.

Since $\cS$ is a convex body and $\s \in \mathring{\cS}$, there exists a ball $\oball(\s,r)$ centered in $\s$ with $r>0$, such that $\oball(\s,r) \subseteq \mathring{\cS}$.
For any $x \in \oball(\s,r) \cap \GG$ is equivalent to $\langle h^\star, x-\s\rangle \geq 0$ since $\langle h^\star, \s \rangle =0$ and $\|x-\s\|\leq r$. %Up to a translation and a rotation, we can consider without loss of generality that $\s = 0$ and $h^\star = (1, 0, \ldots, 0)$. Now, $x = (x_1, \ldots, x_d) \in \oball(\s,r)\cap \GG$ is equivalent to $\|x\| \leq r$ and $x_1 \geq 0$.

Now we define $y_0 = \s + r h^\star /(2 \norm{h^\star})$ and consider the ball $\oball(y_0, r/2)$. For any $y \in \oball(y_0, r/2)$, we can write $y = y_0 + \tilde{y}$ with $\|\tilde{y}\| \leq r/2$. Using $\ps{h^\star}{\s} = 0$, we have $\norm{y - \s} \leq \norm{y_0-\s}+ \norm{\tilde{y}} \leq r $
and $    \langle h^\star, y \rangle = \langle h^\star, y_0 \rangle + \ps{h^\star}{\tilde{y}}$
with $\langle h^\star, y_0 \rangle = r/2$ and $\langle h^\star, \tilde{y} \rangle \geq - \|h^\star\| \cdot \|\tilde{y}\| \geq -1 \cdot r/2 = -r/2$. Therefore $\langle h^\star, y \rangle \geq 0$ and we obtain $y \in \oball(\s,r) \cap \GG$, which gives $\oball(y_0, r/2) \subseteq \oball(\s,r) \cap \GG$.

\end{proof}

\begin{lemma}\label{lem:cyl_properties}
    Given a convex body $\cS \subseteq \R^d$ and a set of orthonormal vectors $\VV = \{v_1, \ldots, v_n\} \subseteq \R^d$, let $\Pi_{\VV^\perp}(\cS)$ be the projection of $\cS$ on the subspace $\VV^\perp = \{ x\in \R^d \mid \langle x, v_i \rangle = 0\} $.
    Define the cylindrification of $S$ onto $\VV$ as
    \begin{align*}
        \Cyl(S, \VV) &\coloneqq \Pi_{\VV^\perp}(\cS) + \Pi_{\text{span}(v_1)}(\cS) + \ldots + \Pi_{\text{span}(v_n)}(\cS) \\
    &=\left\{ x + \sum_{i=1}^n y_i v_i | x \in \Pi_L(\cS), \min_{s \in \cS}\langle s, v_i \rangle \leq y_i \leq \max_{s \in \cS}\langle s, v_i \rangle \right\} \eqsp.
    \end{align*}
    Then it holds that $\cS \subseteq \Cyl(\cS, \VV)$.
\end{lemma}
\begin{proof}
    Define $\Pi_{\VV^\perp}$ as the orthogonal projector on $\VV^\perp$. Then we have for any $s \in \cS$
    \begin{equation*}
        s = \Pi_{\VV^\perp}s + (I - \Pi_{\VV^\perp})s \eqsp,
    \end{equation*}
    where $(I - \Pi_{\VV^\perp})$ is an orthogonal projector on the space $\text{span}(\{v_1,\ldots,v_n\})$, thus $(I - \Pi_{\VV^\perp})s = \sum_{i=1}^n y_i v_i$ for $y_i = \langle s, v_i \rangle$. This decomposition allows us to conclude.
\end{proof}

\begin{algorithm}[!ht]
\caption{$\pvaU$}\label{algorithm:projected_volume}
\begin{algorithmic}[1]
    \STATE {\bfseries Input:} $T, \pvacst, \cS_t \text{ such that } \diam(\cS_t) \geq \pvacst, \VV_t$, $a_t^1, a_t^2$
    \STATE Compute $\Cyl(\cS_t, \VV_t$) and its centroid $\hat{\s}_t$, $ w_t = \fracD{(a_t^1-a_t^2)}{\|a_t^1-a_t^2\|}$, $x_t = \langle \hat{\s}_t, w_t\rangle$, $\pvacst \in \left(0, \fracD{1}{16d(d+1)^2 T^2}\right)$
    \STATE Propose an incentive function $\tsfr(t, a)= 5 \cdot \indi{a_t^1}(a) +(5+\langle \hat{\s}_t,a_t^1-a_t^2\rangle) \cdot  \indi{a_t^2}(a)$
    \IF{$A_t = a_t^1$} 
        \STATE $S_{t+1} = \cS_t\cap \{s|\langle s,w_t\rangle \geq x_t$\}
    \ELSE{} 
        \STATE $S_{t+1} = \cS_t\cap \{s|\langle s,w_t\rangle \leq x_t$\}
    \ENDIF 
    \STATE Let $\VV_{t+1} = \VV_t$% + \ball{0}{\bar{\delta}}$
    \IF{$\exists v \perp \VV_{t+1}$ such that diam$(\cS_{t+1}, v) \leq \pvacst$} 
        \STATE add $v$ to $\VV_t$
        \STATE Repeat this step as many times as necessary
    \ENDIF
    \STATE {\bfseries Output:} $S_{t+1}, \VV_{t+1}$.
\end{algorithmic}
\end{algorithm}
\section{Lower Bound}
\label{appendix:lower_bound}

\begin{proof}[Proof of \Cref{proposition:lower_bound}]
Suppose that the principal was to know $(s_a)_{a \in [K]}$. For any incentive $\icv(t)$ offered on action $a_t$ at round $t$, the agent selects his action following \eqref{eq:def_A_t}: $A_t = \argmax_{a \in [K]} \s_a + \1_{a_t}(a)\icv(t)$. The principal's expected reward is
    \begin{equation*}
        \theta_{A_t} - \1_{a_t}(A_t)\icv(t) \leq \theta_{A_t} - \icvstar_{A_t} = \mu_a \eqsp,
    \end{equation*}
by definition of $\icvstar_a \coloneqq \max_{a' \in [K]} \s_{a'} - \s_a$ as the infimal amount of incentive to be offered on action $a$ to make the agent choose it. Consequently, we have
    \begin{align*}
    \regret(T) & = T\, \mu^\star  - \sum_{t=1}^T \E\parentheseDeux{\theta_{A_t} -\indi{a_t}(A_t) \icv(t) } \\
    &\geq \E\parentheseDeux{\sum_{t=1}^T \mu^\star - (\theta_{A_t} - \icvstar_{A_t})} \\
    & \geq \E\parentheseDeux{\sum_{t=1}^T \mu^\star - \mu_{A_t} }\eqsp.
\end{align*}
Assuming the principal knows $(\s_a)_{a \in [K]}$, observing $X_a(t)$ is equivalent to observing $X_a(t) - \icvstar_a$. Using the result of \citet{burnetas1996optimal} \citep[see, e.g.,][Theorem 16.2.]{lattimore2020bandit}, it then comes
\begin{align*}
\liminf_{T \to \infty} \frac{\regret(T)}{\log T} &\geq \sum_{a,\mu_a<\mu^{\star}} \frac{\mu^{\star}-\mu_a}{\mathrm{KL}_{\inf}(\nu_a - \icvstar_a,\mu^{\star},\cD)} \eqsp.
\end{align*}
\end{proof}

\end{document}